\documentclass{article}

\usepackage[nonatbib,final]{neurips_2019}
\usepackage[numbers]{natbib}

\usepackage[utf8]{inputenc} 
\usepackage[T1]{fontenc}    
\usepackage{hyperref}       
\usepackage{url}            
\usepackage{booktabs}       
\usepackage{amsfonts}       
\usepackage{nicefrac}       
\usepackage{microtype}      

\usepackage{algorithm}
\usepackage{algorithmic}
\usepackage{subcaption}
\usepackage{amsmath}
\usepackage{amssymb}
\usepackage{amsthm}
\usepackage{mathtools}
\usepackage{natbib}
\usepackage{enumerate}
\usepackage{tikz}
\usepackage{graphicx}

\newtheorem{theorem}{Theorem}[section]

\newtheorem{definition}{Definition}[section]
\newtheorem{assumption}{Assumption}[section]

\newcommand{\norminf}[1]{\left\lVert#1\right\rVert_\infty}

\newcommand{\defeq}{\mathrel{\stackrel{\makebox[0pt]{\mbox{\normalfont\tiny def}}}{=}}}

\newcommand{\trans}{P}
\newcommand{\backup}{\mathcal{T}}
\newcommand{\Qclass}{\mathcal{Q}}

\newcommand{\linfnorm}{{\ell_\infty}}

\newcommand{\Tpi}{\backup^\pi}
\newcommand{\TPi}{\backup^\Pi}
\newcommand{\VPi}{V^\Pi}

\newcommand{\dataset}{\mathcal{D}}
\newcommand{\expec}{\mathbb{E}}
\newcommand{\rhoinit}{{\rho_0}}
\newcommand{\Pieps}{{\Pi_\epsilon}}

\newcommand{\projerr}{\delta}

\newcommand{\valerr}{\zeta}

\newcommand{\mmd}{\operatorname{MMD}}

\usepackage{wrapfig}

\title{Stabilizing Off-Policy Q-Learning via Bootstrapping Error Reduction}

\newcommand*\samethanks[1][\value{footnote}]{\footnotemark[#1]}

\author{%
    Aviral Kumar\thanks{Equal Contribution}\\
    UC Berkeley\\
    \texttt{aviralk@berkeley.edu}\\
    \And
    Justin Fu\samethanks\ \\
    UC Berkeley\\
    \texttt{justinjfu@eecs.berkeley.edu}\\
    \AND
    George Tucker\\
    Google Brain\\
    \texttt{gjt@google.com}\\
    \And
    Sergey Levine\\
    UC Berkeley, Google Brain\\
    \texttt{svlevine@eecs.berkeley.edu}
}

\begin{document}

\maketitle

\begin{abstract}
Off-policy reinforcement learning aims to leverage experience collected from prior policies for sample-efficient learning. However, in practice, commonly used off-policy approximate dynamic programming methods based on Q-learning and actor-critic methods are highly sensitive to the data distribution, and can make only limited progress without collecting additional on-policy data. As a step towards more robust off-policy algorithms, we study the setting where the off-policy experience is fixed and there is no further interaction with the environment. We identify \emph{bootstrapping error} as a key source of instability in current methods. Bootstrapping error is due to bootstrapping from actions that lie outside of the training data distribution, and it accumulates via the Bellman backup operator. We theoretically analyze bootstrapping error, and demonstrate how carefully constraining action selection in the backup can mitigate it. Based on our analysis, we propose a practical algorithm, bootstrapping error accumulation reduction (BEAR). We demonstrate that BEAR is able to learn robustly from different off-policy distributions, including random and suboptimal demonstrations, on a range of continuous control tasks.

\end{abstract}

\section{Introduction}
\label{sec:intro}
\vspace{-5pt}

One of the primary drivers
of the success of machine learning methods in open-world perception settings, such as computer vision~\cite{he2016resnet} and NLP~\cite{devlin2018bert}, has been the ability of high-capacity function approximators, such as deep neural networks, to learn generalizable models from large amounts of data. Reinforcement learning (RL) has proven comparatively difficult to scale to unstructured real-world settings because most RL algorithms require active data collection. As a result, RL algorithms can learn complex behaviors in simulation, where data collection is straightforward, 
but real-world performance is limited by the expense of active data collection. 
In some domains, such as autonomous driving~\cite{yu2018bdd} and recommender systems~\citep{bennett2007netflix}, previously collected datasets are plentiful. Algorithms that can utilize such datasets effectively would not only make real-world RL more practical, but also would enable substantially better generalization by incorporating diverse prior experience.  

In principle, off-policy RL algorithms can leverage this data; however, in practice, off-policy algorithms are limited in their ability to learn entirely from off-policy data. 
Recent off-policy RL methods   (e.g.,~\citep{haarnoja2018sac,munos2016safe,kalashnikov18qtopt,impala2018}) have demonstrated sample-efficient performance on complex tasks in robotics~\cite{kalashnikov18qtopt} and simulated environments~\cite{mujoco}. 
However, these methods can still fail to learn when presented with arbitrary off-policy data without the opportunity to collect more experience from the environment. This issue persists even when the off-policy data comes from effective expert policies, which in principle should address any exploration challenge~\citep{deBruin2015importance,fujimoto2018off,fu2019diagnosing}. This sensitivity to the training data distribution is a limitation of practical off-policy RL algorithms, and one would hope that an off-policy algorithm should be able to learn reasonable policies through training on static datasets before being deployed in the real world. 
In this paper, we aim to develop off-policy, value-based RL methods that can learn from large, static datasets. As we show, a crucial challenge in applying value-based methods to off-policy scenarios arises in the bootstrapping process employed
when Q-functions are evaluated on out of \textit{out-of-distribution} action inputs for computing the backup when training from off-policy data. This may introduce errors in the Q-function and the algorithm is unable to collect new data in order to remedy those errors, making training unstable and potentially diverging. 
Our primary contribution is an analysis of error accumulation in the bootstrapping process due to out-of-distribution inputs and a practical way of addressing this error. 
First, we formalize and analyze the reasons for instability and poor performance when learning from off-policy data. We show that, through careful action selection, error propagation through the Q-function can be mitigated. We then propose a principled algorithm called \emph{bootstrapping error accumulation reduction} (BEAR) to control bootstrapping error in practice, which uses the notion of \emph{support-set matching} to prevent error accumulation. Through systematic experiments, we show the effectiveness of our method on continuous-control MuJoCo tasks, with a variety of off-policy datasets: generated by a random, suboptimal, or optimal policies. BEAR is consistently robust to the training dataset, matching or exceeding the state-of-the-art in all cases, whereas existing algorithms only perform well for specific datasets.

\vspace{-10pt}
\section{Related Work}
\label{sec:related}
\vspace{-10pt}
In this work, we study off-policy reinforcement learning with static datasets. Errors arising from inadequate sampling, distributional shift, and function approximation have been rigorously studied as ``error propagation'' in approximate dynamic programming (ADP)~\citep{bertsekas1996ndp,munos2003errorapi,farahmand2010error,bruno2015approximate}. These works often study how Bellman errors accumulate and propagate to nearby states via bootstrapping. In this work, we build upon tools from this analysis to show that performing Bellman backups on static datasets leads to error accumulation due to out-of-distribution values. Our approach is motivated as reducing the rate of propagation of error propagation between states.

Our approach constrains actor updates so that the actions remain in the support of the training dataset distribution. Several works have explored similar ideas in the context of off-policy learning learning in online settings. \citet{kakade2002cpi} shows that large policy updates can be destructive, and propose a conservative policy iteration scheme which constrains actor updates to be small for provably convergent learning. \citet{grau-moya2018soft} use a learned prior over actions in the maximum entropy RL framework~\citep{levine2018rlasinference} and justify it as a regularizer based on mutual information. However, none of these methods use static datasets. Importance Sampling based distribution re-weighting~\cite{munos2016safe,gelada2019off,precup2001offpol,mahmood2015emphatic} has also been explored primarily in the context of off-policy policy evaluation.

Most closely related to our work is batch-constrained Q-learning (BCQ)~\citep{fujimoto2018off} and SPIBB~\citep{laroche2019spibb},
which also discuss instability arising from previously unseen actions. \citet{fujimoto2018off} show convergence properties of an action-constrained Bellman backup operator in tabular, error-free settings. We prove stronger results under approximation errors and provide a bound on the \emph{suboptimality} of the solution. This is crucial as it drives the design choices for a practical algorithm. 
As a consequence, although we experimentally find that \citep{fujimoto2018off} outperforms standard Q-learning methods when the off-policy data is collected by an expert, BEAR outperforms \cite{fujimoto2018off} when the off-policy data is collected by a suboptimal policy, as is common in real-life applications. Empirically, we find BEAR  achieves stronger and more consistent results than BCQ across a wide variety of datasets and environments. As we explain below, the BCQ constraint is too aggressive;  BCQ generally fails to substantially improve over the behavior policy, while our method actually improves when the data collection policy is suboptimal or random. SPIBB~\citep{laroche2019spibb}, like BEAR, is an algorithm based on constraining the learned policy to the support of a behavior policy. However, the authors do not extend safe performance guarantees from the batch-constrained case to the relaxed support-constrained case, and do not evaluate on high-dimensional control tasks. REM~\citep{agarwal19striving} is a concurrent work that uses a random convex combination of an ensemble of Q-networks to perform offline reinforcement learning from a static dataset consisting of interaction data generated while training a DQN agent.

\vspace{-10pt}
\section{Background}
\label{sec:background}
\vspace{-0.1in}
We represent the environment as a Markov decision process (MDP) defined by a tuple $(\mathcal{S}, \mathcal{A}, \trans, R, \rhoinit, \gamma)$, where $\mathcal{S}$ is the state space, $\mathcal{A}$ is the action space, $\trans(s' | s, a)$ is the transition distribution, $\rhoinit(s)$ is the initial state distribution, $R(s,a)$ is the reward function, and $\gamma \in (0,1)$ is the discount factor. The goal in RL is to find a policy $\pi(a|s)$ that maximizes the expected cumulative discounted rewards which is also known as the return. The notation $\mu_\pi(s)$ denotes the discounted state marginal of a policy $\pi$, defined as the average state visited by the policy, $\sum_{t=0}^\infty \gamma^t p^t_\pi(s)$. $\trans^\pi$ is shorthand for the transition matrix from $s$ to $s'$ following a certain policy $\pi$, $p(s'|s) = E_{\pi}[p(s'|s,a)]$.

Q-learning learns the optimal state-action value function 
$Q^*(s,a)$, which represents the expected cumulative discounted reward starting in $s$ taking action $a$ and then acting optimally thereafter. The optimal policy can be recovered from $Q^*$ by choosing the maximizing action. Q-learning algorithms are based on iterating the Bellman optimality operator $\backup$, defined as
\begin{align*}
(\backup \hat{Q})(s, a) \coloneqq R(s, a) + \gamma \expec_{T(s'|s,a)}[\max_{a'}\hat{Q}(s', a')].~~~~~
\end{align*}
When the state space is large, we represent $\hat{Q}$ as a hypothesis from the set of function approximators $\Qclass$ (e.g., neural networks). In theory, the estimate of the $Q$-function is updated by projecting $\backup \hat{Q}$ into $\Qclass$ (i.e., minimizing the mean squared Bellman error $\expec_{\nu}[( Q - \backup \hat{Q})^2]$, where $\nu$ is the state occupancy measure under the behaviour policy). This is also referred to a \emph{Q-iteration}. In practice, an empirical estimate of $\backup \hat{Q}$ is formed with samples, and treated as a supervised $\ell_2$ regression target to form the next approximate $Q$-function iterate. 

In large action spaces (e.g., continuous), the maximization $\max_{a'} Q(s', a')$
is generally intractable. Actor-critic methods~\cite{suttonrlbook,fujimoto18addressing,haarnoja2018sac} address this by additionally learning a policy $\pi_{\theta}$ that maximizes the $Q$-function. 
In this work, we study off-policy learning from a static dataset of transitions $\dataset = \{(s, a, s', R(s, a))\}$, collected under an unknown behavior policy $\beta(\cdot|s)$. We denote the distribution over states and actions induced by $\beta$ as $\mu(s,a)$.

\section{Out-of-Distribution Actions in Q-Learning}
\label{sec:Problem Description}


\begin{wrapfigure}{r}{0.5\textwidth}
\vspace{-10pt}
\begin{center}
\vspace{-0.1in}
    \includegraphics[width=0.48\linewidth]{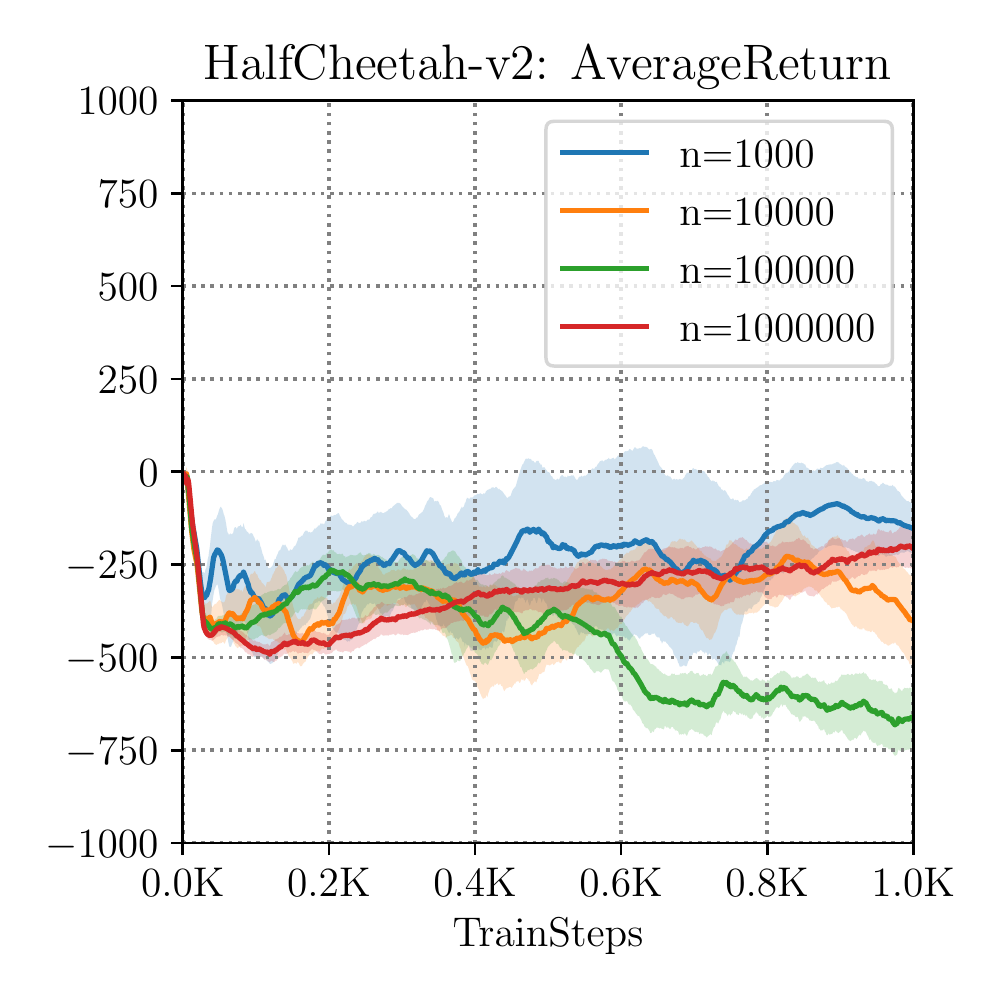}
    ~
    \includegraphics[width=0.48\linewidth]{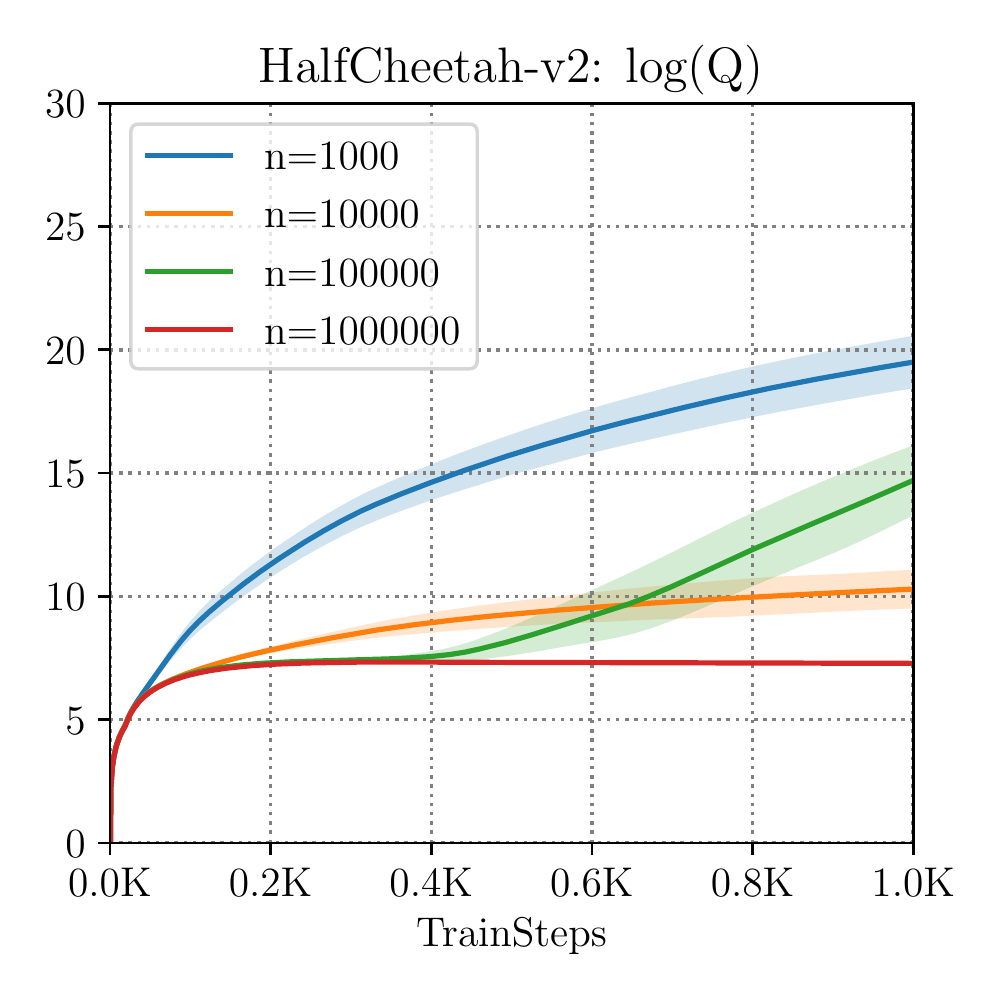}
  \end{center}
 \vspace{-10pt}
  \caption{ \footnotesize Performance of SAC on HalfCheetah-v2 (return (left) and $\log$ Q-values (right)) with off-policy expert data w.r.t. number of training samples ($n$). Note the large discrepancy between returns (which are negative) and $\log$ Q-values (which have large positive values), which is not solved with additional samples.} 
 \vspace{-15pt}
 \label{fig:divergence}
\end{wrapfigure}
Q-learning methods often fail to learn on static, off-policy data, as shown in Figure \ref{fig:divergence}. At first glance, this resembles overfitting, but increasing the size of the static dataset does not rectify the problem, suggesting the issue is more complex. We can understand the source of this instability by examining the form of the Bellman backup. Although minimizing the mean squared Bellman error corresponds to a supervised regression problem, the targets for this regression are themselves derived from the current Q-function estimate. The targets are calculated by maximizing the learned $Q$-values with respect to the action at the next state. However, the $Q$-function estimator is only reliable on inputs from the same distribution as its training set. As a result, na\"{i}vely maximizing the value may evaluate the $\hat{Q}$ estimator on actions that lie far outside of the training distribution, resulting in pathological values that incur large error. We refer to these actions as out-of-distribution (OOD) actions. 

Formally, let $\valerr_k(s, a) = |Q_k(s,a) - Q^*(s,a)|$ denote the total error at iteration $k$ of Q-learning, and let $\projerr_k(s, a) = |Q_k(s,a) - \backup Q_{k-1}(s,a)|$ denote the current Bellman error. Then, we have \mbox{$\valerr_k(s, a) \le \projerr_k(s,a) + \gamma \max_{a'} \expec_{s'}[\valerr_{k-1}(s',a')]$}. In other words, errors from $(s', a')$ are discounted, then accumulated with new errors $\projerr_k(s, a)$ from the current iteration. We expect $\projerr_k(s,a)$ to be high on OOD states and actions, as errors at these state-actions are never directly minimized while training.

To mitigate bootstrapping error, we can restrict the policy to ensure that it output actions that lie in the support of the training distribution. This is distinct from previous work (e.g., BCQ~\citep{fujimoto2018off}) which implicitly constrains the \emph{distribution} of the learned policy to be close to the behavior policy, similarly to behavioral cloning~\cite{Schaal99isimitation}.
While this is sufficient to ensure that actions lie in the training set with high probability, it is overly restrictive. For example, if the behavior policy is close to uniform, the learned policy will behave randomly, resulting in poor performance, even when the data is sufficient to learn a strong policy (see Figure~\ref{fig:gridworld}
for an illustration). {Formally, this means that a learned policy $\pi(a|s)$ has positive density\textit{ only where} the density of the behaviour policy $\beta(a|s)$ is more than a threshold (i.e., $\forall a, \beta(a|s) \leq \varepsilon \implies \pi(a|s) = 0$), instead of a closeness constraint on the value of the density $\pi(a|s)$ and $\beta(a|s)$.}
Our analysis instead reveals a tradeoff between staying within the data distribution and finding a suboptimal solution when the constraint is too restrictive. Our analysis motivates us to restrict the support of the learned policy, but not the probabilities of the actions lying within the support. This avoids evaluating the Q-function estimator on OOD actions, but remains flexible in order to find a performant policy. Our proposed algorithm leverages this insight. 

\subsection{Distribution-Constrained Backups}
\label{sec:dist_constrained}
In this section, we define and analyze a backup operator that restricts the set of policies used in the maximization of the Q-function, and we derive performance bounds which depend on the restricted set. This provides motivation for constraining policy support to the data distribution. We begin with the definition of a distribution-constrained operator:
\begin{definition}[Distribution-constrained operators]
Given a set of policies $\Pi$
, the distribution-constrained backup operator is defined as:
\begin{align*}
\TPi Q(s, a) \defeq \expec \big[ R(s, a) + \gamma \max_{\pi \in \Pi} \expec_{\trans(s' | s, a)}\left[V(s') \right] \big]
\ \ \ \ \ \ \ \ \ \ \ \ 
V(s) \defeq \max_{\pi \in \Pi} \expec_{\pi}[Q(s, a)]\ \ .
\end{align*}
\end{definition}
\vspace{-5pt}
This backup operator satisfies properties of the standard Bellman backup, such as convergence to a fixed point, as discussed in Appendix~\ref{app:constrained_backup}. To analyze the (sub)optimality of performing this backup under approximation error, we first quantify two sources of error. The first is a \emph{suboptimality bias}. The optimal policy may lie outside the policy constraint set, and thus a suboptimal solution will be found. The second arises from distribution shift between the training distribution and the policies used for backups. This formalizes the notion of OOD actions. 
To capture suboptimality in the final solution, we define a \emph{suboptimality constant}, which measures how far $\pi^*$ is from $\Pi$. 
\begin{definition}[Suboptimality constant]
The suboptimality constant is defined as:
\[ \alpha(\Pi) = \max_{s,a} |\TPi Q^*(s, a) - \backup Q^*(s, a)|. \]
\end{definition}
\vspace{-10pt}
Next, we define a concentrability coefficient~\citep{munos2005erroravi}, which quantifies how far the visitation distribution generated by policies from $\Pi$ is  from the training data distribution. This constant captures the degree to which states and actions are out of distribution.
\begin{assumption}[Concentrability]
Let $\rhoinit$ denote the initial state distribution, and $\mu(s, a)$ denote the distribution of the training data over $\mathcal{S} \times \mathcal{A}$, with marginal $\mu(s)$ over $\mathcal{S}$. Suppose there exist coefficients $c(k)$ such that for any $\pi_1, ... \pi_k \in \Pi$ and $s \in \mathcal{S}$:
\[
\rhoinit P^{\pi_1}P^{\pi_2}...P^{\pi_k}(s) \le c(k) \mu(s),
\]
where $P^{\pi_i}$ is the transition operator on states induced by $\pi_i$.
Then, define the concentrability coefficient $C(\Pi)$ as
\[
C(\Pi) \defeq (1-\gamma)^2\sum_{k=1}^\infty k\gamma^{k-1}c(k).
\] \label{assumption:conc} \end{assumption} 
\vspace{-10pt}
To provide some intuition for $C(\Pi)$, if $\mu$ was generated by a single policy $\pi$, and $\Pi = \{\pi\}$ was a singleton set, then we would have $C(\Pi)=1$, which is the smallest possible value. However, if $\Pi$ contained policies far from $\pi$, the value could be large, potentially infinite if the support of $\Pi$ is not contained in $\pi$. Now, we bound the performance of approximate distribution-constrained Q-iteration:
\begin{theorem}
\label{thm:avi_bound}
Suppose we run approximate distribution-constrained value iteration with a set constrained backup $\TPi$. Assume that $\delta(s,a) \ge \max_k |Q_k(s,a) - \TPi Q_{k-1}(s,a)|$ bounds the Bellman error. Then,
\[\lim_{k \to \infty} \expec_{\rhoinit}[|V^{\pi_k}(s) - V^*(s)|] \le
\frac{\gamma}{(1-\gamma)^2}\left[ C(\Pi)\expec_\mu[\max_{\pi \in \Pi} \expec_{\pi}[\projerr(s,a)]] + \frac{1-\gamma}{\gamma}\alpha(\Pi) \right]
\]
\end{theorem}
\begin{proof} See Appendix~\ref{app:error_prop}, Theorem~\ref{thm:avi_bound_proof} \end{proof}
This bound formalizes the tradeoff between keeping policies chosen during backups close to the data (captured by $C(\Pi)$) and keeping the set $\Pi$ large enough to capture well-performing policies (captured by $\alpha(\Pi)$). When we expand the set of policies $\Pi$, we are increasing $C(\Pi)$ but decreasing $\alpha(\Pi)$. An example of this tradeoff, and how a careful choice of $\Pi$ can yield superior results, is given in a tabular gridworld example in Fig.~\ref{fig:gridworld}, where we visualize errors accumulated during distribution-constrained Q-iteration for different choices of $\Pi$. 

Finally, we motivate the use of support sets to construct $\Pi$. We are interested in the case where $\Pi_\epsilon = \{ \pi ~|~ \pi( a | s) = 0 \text{ whenever } \beta( a | s) < \epsilon \}$, where $\beta$ is the behavior policy (i.e., $\Pi$ is the set of policies that have support in the probable regions of the behavior policy). Defining $\Pi_\epsilon$ in this way allows us 
to bound the concentrability coefficient:
\begin{theorem}
\label{thm:conc_coeff_bound}
Assume the data distribution $\mu$ is generated by a behavior policy $\beta$. 
Let $\mu(s)$ be the marginal state distribution under the data distribution. Define $\Pieps = \{ \pi ~|~ \pi( a | s) = 0 \text{ whenever } \beta( a | s) < \epsilon \}$ and let $\mu_\Pieps$ be the highest discounted marginal state distribution starting from the initial state distribution $\rho$ and following policies $\pi \in \Pieps$ at each time step thereafter. Then, there exists a concentrability coefficient $C(\Pieps)$ which is bounded:
\[
C(\Pi_\epsilon) \leq C(\beta) \cdot \Big(1 + \frac{\gamma}{(1 - \gamma) f(\epsilon)} (1 - \epsilon)\Big)
\]
where $f(\epsilon) \defeq \min_{s \in \mathcal{S}, \mu_\Pieps(s) > 0} [\mu(s)] > 0$.
\end{theorem}
\vspace{-10pt}
\begin{proof} See Appendix~\ref{app:error_prop}, Theorem~\ref{thm:conc_coeff_proof} \end{proof}
\vspace{-10pt}
Qualitatively, $f(\epsilon)$ is the minimum discounted visitation marginal of a state under the behaviour policy if only actions which are more than $\epsilon$ likely are executed in the environment. Thus, using support sets gives us a single lever, $\epsilon$, which simultaneously trades off the value of $C(\Pi)$ and $\alpha(\Pi)$. Not only can we provide theoretical guarantees, we will see in our experiments (Sec.~\ref{sec:experiments}) that constructing $\Pi$ in this way provides a simple and effective method for implementing distribution-constrained algorithms. 

Intuitively, this means we can prevent an increase in overall error in the Q-estimate by selecting policies supported on the support of the training action distribution, which would ensure roughly bounded projection error $\delta_k(s, a)$ while reducing the suboptimality bias, potentially by a large amount. Bounded error $\delta_k(s, a)$ on the support set of the training distribution is a reasonable assumption when using highly expressive function approximators, such as deep networks, especially if we are willing to reweight the transition set~\cite{Schaul2016PrioritizedER,fu2019diagnosing}. We further elaborate on this point in Appendix~\ref{app:bearql-more}.

\begin{figure}
    \centering
    \vspace{-0.1in}
    \includegraphics[width=0.9\textwidth]{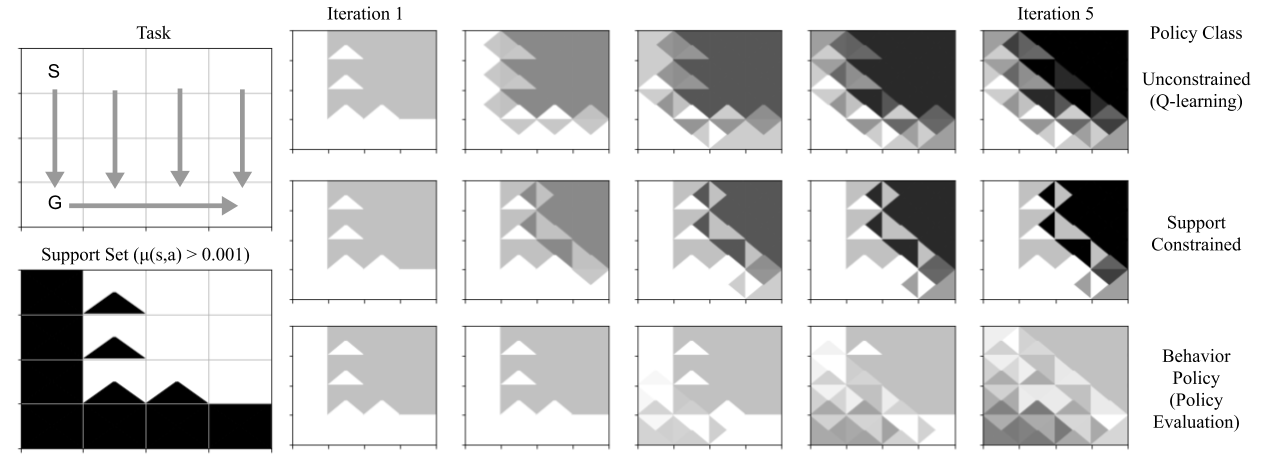}
    \caption{ \footnotesize Visualized error propagation in Q-learning for various choices of the constraint set $\Pi$:
    unconstrained (top row)
    distribution-constrained (middle),
    and constrained to the behaviour policy (policy-evaluation, bottom). Triangles represent Q-values for actions that move in different directions. The task (left) is to reach the bottom-left corner (G) from the top-left (S), but the behaviour policy (visualized as arrows in the task image, support state-action pairs are shown in black on the support set image) travels to the bottom-right with a small amount of $\epsilon$-greedy exploration. Dark values indicate high error, and light values indicate low error. Standard backups propagate large errors from the low-support regions into the high-support regions, leading to high error. Policy evaluation reduces error propagation from low-support regions, but introduces significant suboptimality bias, as the data policy is not optimal. A carefully chosen distribution-constrained backup strikes a balance between these two extremes, by confining error propagation in the low-support region while introducing minimal suboptimality bias.}
    \label{fig:gridworld}
    \vspace{-0.1in}
\end{figure}

\section{Bootstrapping Error Accumulation Reduction (BEAR)}
\label{sec:bear}
\vspace{-0.1in}
We now propose a practical actor-critic algorithm (built on the framework of TD3~\cite{fujimoto18addressing} or SAC~\cite{haarnoja2018sac}) that uses distribution-constrained backups to reduce accumulation of bootstrapping error. The key insight is that we can search for a policy with the same support as the training distribution, while preventing accidental error accumulation.
Our algorithm has two main components. Analogous to BCQ~\citep{fujimoto18addressing}, we use $K$ Q-functions and use the minimum Q-value for policy improvement, and design a constraint which will be used for searching over the set of policies $\Pieps$, which share the same support as the behaviour policy. Both of these components will appear as modifications of the policy improvement step in actor-critic style algorithms. We also note that policy improvement can be performed with the mean of the K Q-functions, and we found that this scheme works as good in our experiments.

We denote the set of Q-functions as: $\hat{Q}_1, \cdots, \hat{Q}_K$.
Then, the policy is updated to maximize the conservative estimate of the Q-values within $\Pieps$: 
$$ \pi_\phi(s) := \max_{\pi \in \Pieps} \expec_{a \sim \pi(\cdot|s)} \left[\min_{j=1,..,K} \hat{Q}_j(s, a)\right] $$
\vspace{-5pt}

In practice, the behaviour policy $\beta$ is unknown, so we need an approximate way to constrain $\pi$ to $\Pi$. We define a differentiable constraint that approximately constrains $\pi$ to $\Pi$, and then approximately solve the constrained optimization problem via dual gradient descent.  We use the sampled version of maximum mean discrepancy (MMD)~\cite{gretton2012kernel}
between the unknown behaviour policy $\beta$ and the actor $\pi$ because it can be estimated based solely on samples from the distributions. Given samples $x_1, \cdots, x_n \sim P$ and $y_1, \cdots, y_m \sim Q$, the sampled MMD between $P$ and $Q$ is given by:\\
$$\operatorname{MMD}^2(\{x_1, \cdots, x_n\}, \{y_1, \cdots, y_m\}) = \frac{1}{n^2} \sum_{i, i'} k(x_i, x_{i'}) - \frac{2}{nm} \sum_{i, j} k(x_i, y_j) + \frac{1}{m^2} \sum_{j, j'} k(y_j, y_{j'}).
$$
Here, $k(\cdot, \cdot)$ is any universal kernel. In our experiments, we find both Laplacian and Gaussian kernels work well.
The expression for MMD does not involve the density of either distribution and it can be optimized directly through samples. Empirically we find that, in the low-intermediate sample regime, the sampled MMD between $P$ and $Q$ is similar to the MMD between a uniform distribution over $P$'s support and $Q$, which makes MMD roughly suited for constraining distributions to a given support set. (See Appendix~\ref{app:mmd} for numerical simulations justifying this approach).

Putting everything together, the optimization problem in the policy improvement step is
\vspace{-5pt}
\begin{multline}
    \label{eqn:policy_update}
   \pi_\phi := \max_{\pi \in \Delta_{|S|}} \expec_{s \sim \mathcal{D}} \expec_{a \sim \pi(\cdot|s)} \left[\min_{j=1,..,K} \hat{Q}_j(s, a)\right] 
   \text{~~s.t.~~} \mathbb{E}_{s \sim \mathcal{D}} [\operatorname{MMD}(\mathcal{D}(s), \pi(\cdot|s))] \leq \varepsilon \quad
\end{multline}
where $\varepsilon$ is an approximately chosen threshold. We choose a threshold of $\varepsilon=0.05$ in our experiments. The algorithm is summarized in Algorithm~\ref{algo:bear_ql}. 

{How does BEAR connect with distribution-constrained backups described in Section 4.1? Step 5 of the algorithm restricts $\pi_\phi$ to lie in the support of $\beta$. This insight is formally justified in Theorems 4.1 \& 4.2 ($C(\Pi_\varepsilon)$ is bounded). Computing distribution-constrained backup exactly by maximizing over $\pi \in \Pi_\varepsilon$ is intractable in practice. As an approximation, we sample Dirac policies in the support of $\beta$ (Alg 1, Line 5) and perform empirical maximization to compute the backup. As the maximization is performed over a \textit{narrower} set of Dirac policies ($\{ \delta_{a_i} \} \subseteq \Pi_\varepsilon$), the bound in Theorem 4.1 still holds. Empirically, we show in Section~\ref{sec:experiments} that this approximation is sufficient to outperform previous methods. This connection is briefly discussed in Appendix C.2.}

\vspace{-5pt}
\begin{algorithm}[H]
\small
\caption{BEAR Q-Learning (BEAR-QL)}
\label{alg:q_learning}
\begin{algorithmic}[1]
    \INPUT: Dataset $\mathcal{D}$, target network update rate $\tau$, mini-batch size $N$, sampled actions for MMD $n$, minimum $\lambda$
    \STATE Initialize Q-ensemble $\{Q_{\theta_i} \}_{i=1}^{K}$, actor $\pi_{\phi}$, Lagrange multiplier $\alpha$, target networks $\{ Q_{\theta'_i} \}_{i=1}^K$, and a target actor $\pi_{\phi'}$, with $\phi' \leftarrow \phi, \theta'_i \leftarrow \theta_i$
    \FOR{$t$ in \{1, \dots, N\}}
        \STATE Sample mini-batch of transitions $(s, a, r, s') \sim \mathcal{D}$\\
        \textbf{Q-update:}
            \STATE Sample $p$ action samples, $\{a_i \sim \pi_{\phi'}(\cdot|s')\}_{i=1}^p$
            \STATE Define $y(s, a) := \max_{a_i} [ \lambda \min_{j=1,..,K} Q_{\theta'_j}(s', a_i) + (1 - \lambda) \max_{j=1,..,K} Q_{\theta'_j}(s', a_i)]$
            \STATE $\forall i, \theta_i \leftarrow \arg \min_{\theta_i} (Q_{\theta_i}(s, a) - (r + \gamma y(s, a)))^2$\\
        \textbf{Policy-update:}
        \STATE Sample actions $\{ \hat{a}_i \sim \pi_{\phi}(\cdot | s) \}_{i=1}^{m}$ and $\{ a_j \sim \mathcal{D}(s)\}_{j=1}^{n}$, $n$ preferably an intermediate integer(1-10)
        \STATE Update $\phi$, $\alpha$ by minimizing Equation~\ref{eqn:policy_update} by using dual gradient descent with Lagrange multiplier $\alpha$
        \STATE \textbf{Update Target Networks: } $\theta'_i \leftarrow \tau \theta_i + (1 - \tau)\theta'_i$; $\phi' \leftarrow \tau \phi + (1 -\tau) \phi'$ 
    \ENDFOR
\end{algorithmic}
\label{algo:bear_ql}
\end{algorithm}
\vspace{-0.1in}
In summary, the actor is updated towards maximizing the Q-function while still being constrained to remain in the valid search space defined by $\Pieps$. The Q-function uses actions sampled from the actor to then perform distribution-constrained Q-learning, over a reduced set of policies. {At test time, we sample $p$ actions from $\pi_\phi(s)$ and the Q-value maximizing action out of these is executed in the environment.}  
Implementation and other details are present in Appendix \ref{app:additional_details}.

\section{Experiments}
\vspace{-0.1in}
\label{sec:experiments}
In our experiments, we study how BEAR performs when learning from static off-policy data on a variety of continuous control benchmark tasks. We evaluate our algorithm in three settings: when the dataset $\dataset$ is generated by \textbf{(1)} a completely random behaviour policy, \textbf{(2)} a partially trained, medium scoring policy, and \textbf{(3)} an optimal policy. Condition \textbf{(2)} is of particular interest, as it captures many common use-cases in practice, such as learning from imperfect demonstration data (e.g., of the sort that are commonly available for autonomous driving~\cite{DBLP:conf/iclr/GaoXLYLD18}), or reusing previously collected experience during off-policy RL. We compare our method to several prior methods: a baseline actor-critic algorithm (TD3), the BCQ  algorithm~\citep{fujimoto2018off}, which aims to address a similar problem, as discussed in Section~\ref{sec:Problem Description}, KL-control~\citep{jacques19way} (which solves a KL-penalized RL problem similarly to maximum entropy RL), a static version of DQfD~\citep{hester2018dqfd} (where a constraint to upweight Q-values of state-action pairs observed in the dataset is added as an auxiliary loss on top a regular actor-critic algorithm), and a behaviour cloning (BC) baseline, which simply imitates the data distribution. This serves to measure whether each method actually performs effective RL, or simply copies the data. We report the average evaluation return over 5 seeds of the policy given by the learned algorithm, in the form of a learning curve as a function of number of gradient steps taken by the algorithm. These samples are only collected for evaluation, and are not used for training.
\vspace{-0.1in}

\subsection{Performance on Medium-Quality Data}

We first discuss the evaluation of condition with ``mediocre'' data \textbf{(2)}, as this condition resembles the settings where we expect training on offline data to be most useful. We collected one million transitions from a partially trained policy, so as to simulate imperfect demonstration data or data from a mediocre prior policy.
In this scenario, we found that BEAR-QL consistently outperforms both BCQ~\cite{fujimoto2018off} and a na\"ive off-policy RL baseline (TD3) by large margins, as shown in Figure~\ref{fig:mediocre}. This scenario is the most relevant from an application point of view, as access to optimal data may not be feasible, and random data might have inadequate exploration to efficient learn a good policy. We also evaluate the accuracy with which the learned Q-functions predict actual policy returns. These trends are provided in Appendix~\ref{app:q_vs_mc}. Note that the performance of BCQ often tracks the performance of the BC baseline, suggesting that BCQ primarily imitates the data. Our KL-control baseline uses automatic temperature tuning~\citep{haarnoja2018sac}. We find that KL-control usually performs similar or worse to BC, whereas DQfD tends to diverge often due to cumulative error due to OOD actions and often exhibits a huge variance across different runs (for example, HalfCheetah-v2 environment). 

\begin{figure*}[t!]
    \centering
    \vspace{-0.05in}
    \begin{subfigure}[t]{0.23\textwidth}
        \centering
        \includegraphics[width=0.99\linewidth]{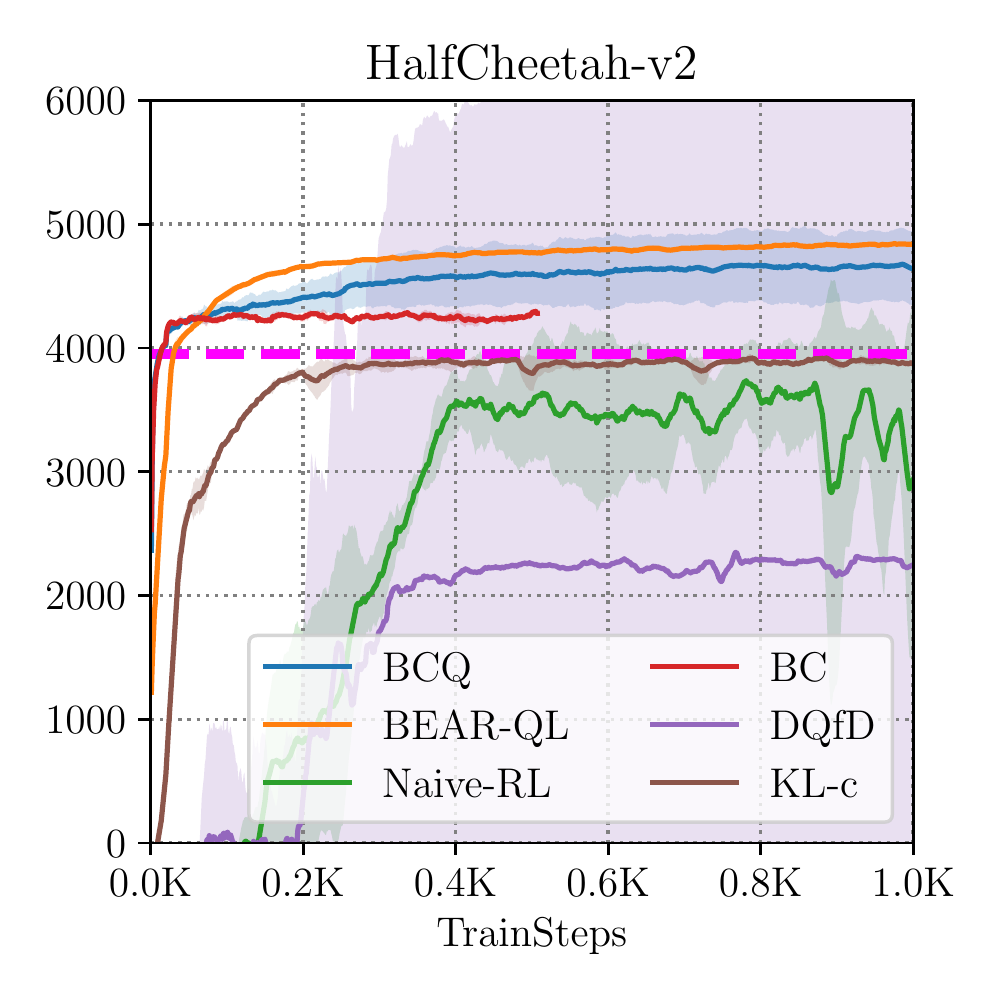}
    \end{subfigure}
    \begin{subfigure}[t]{0.23\textwidth}
        \centering
        \includegraphics[width=0.99\linewidth]{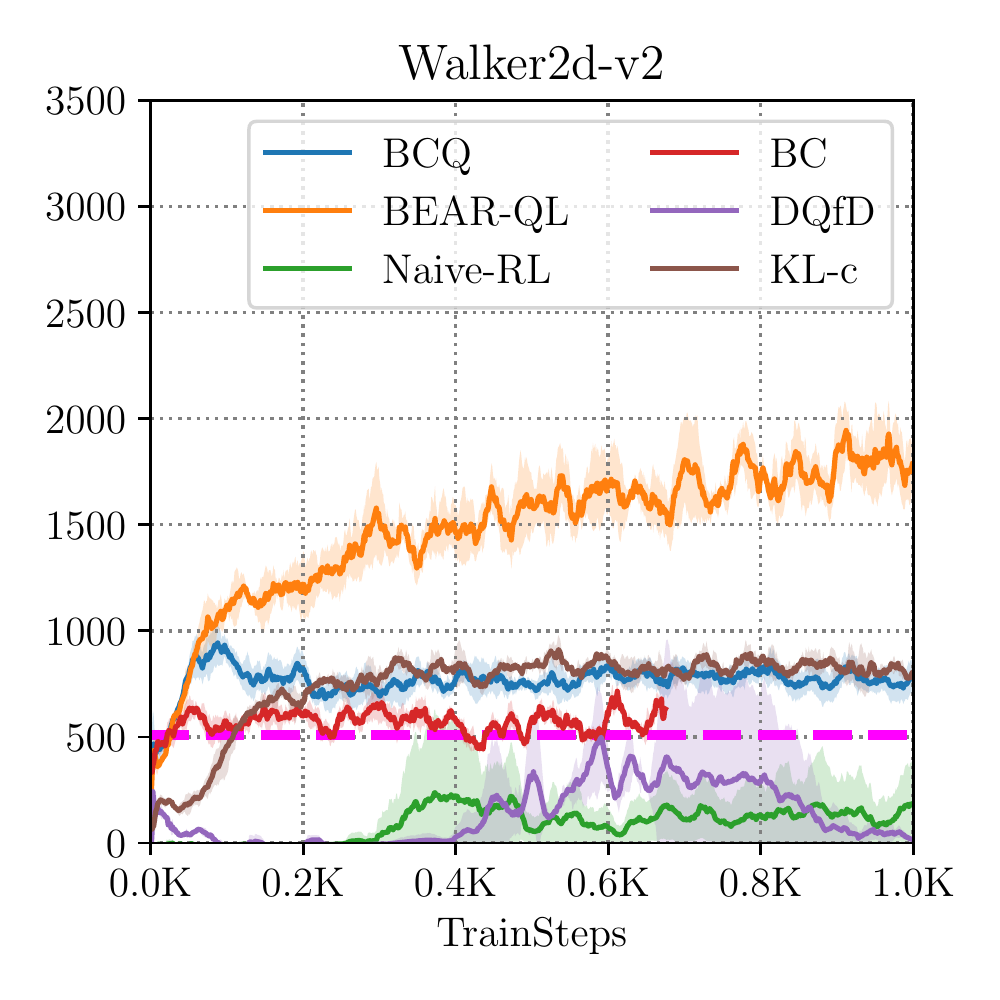}
    \end{subfigure}
    ~
    \begin{subfigure}[t]{0.23\textwidth}
        \centering
        \includegraphics[width=0.99\linewidth]{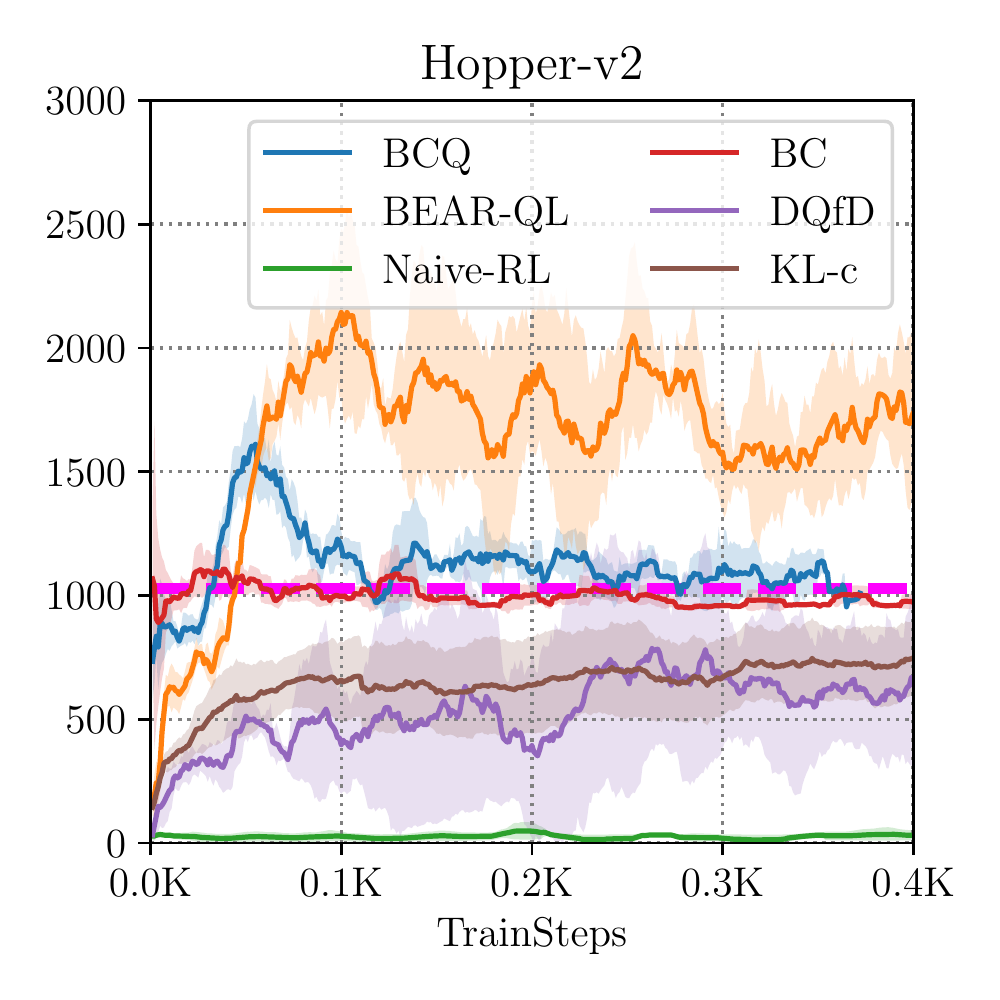}
    \end{subfigure}
    ~
    \begin{subfigure}[t]{0.23\textwidth}
        \centering
        \includegraphics[width=0.99\linewidth]{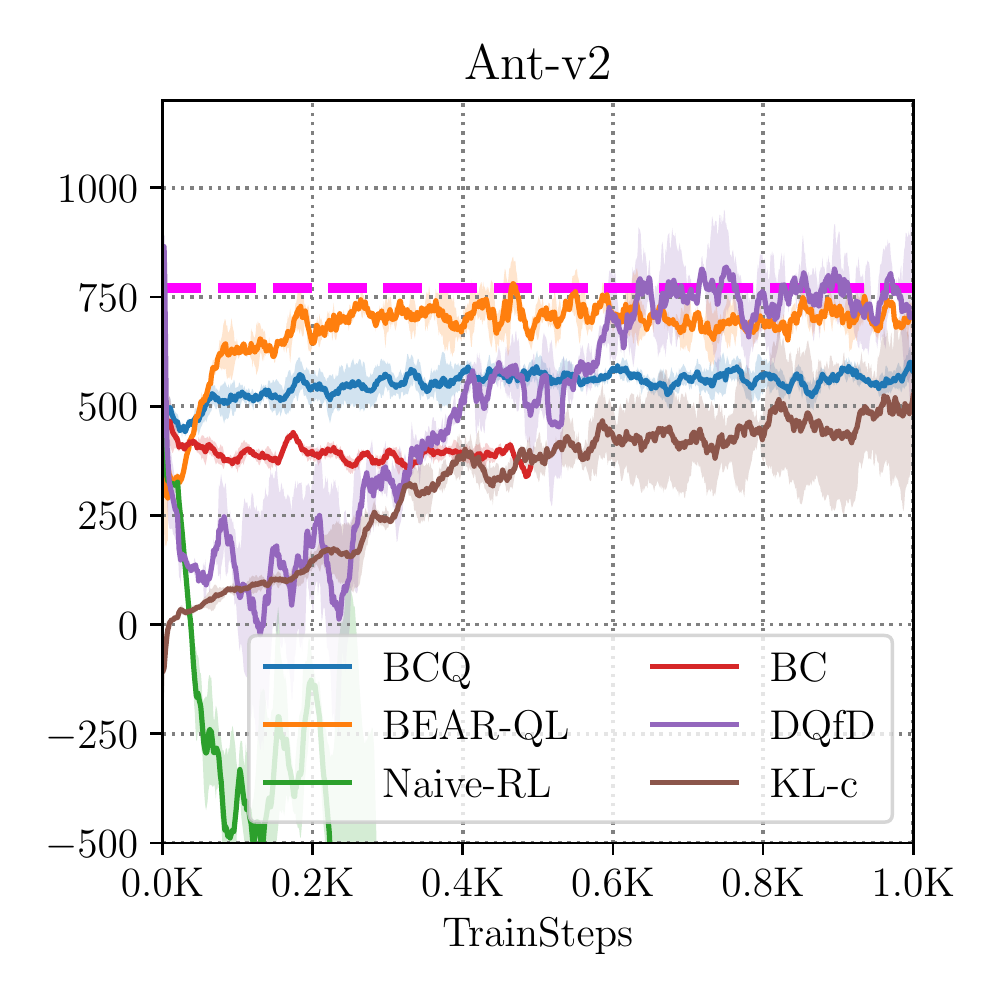}
    \end{subfigure}
    \caption{ \footnotesize Average performance of BEAR-QL, BCQ, Na\"ive RL and BC on medium-quality data averaged over 5 seeds. BEAR-QL outperforms both BCQ and Na\"ive RL. Average return over the training data is indicated by the magenta line. One step on the x-axis corresponds to 1000 gradient steps.}
    \label{fig:mediocre}
    \vspace{-0.1in}
\end{figure*}


\vspace{-5pt}
\subsection{Performance on Random and Optimal Datasets}
In Figure~\ref{fig:optimal_random}, we show the performance of each method when trained on data from a random policy (top) and a near-optimal policy (bottom). In both cases, our method BEAR achieves good results, consistently exceeding the average dataset return on random data, and matching the optimal policy return on optimal data. Na\"{i}ve RL also often does well on random data. For a random data policy, all actions are in-distribution, since they all have equal probability. This is consistent with our hypothesis that OOD actions are one of the main sources of error in off-policy learning on static datasets. The prior BCQ method~\cite{fujimoto2018off} performs well on optimal data but performs poorly on random data, where the constraint is too strict. These results show that BEAR-QL is robust to the dataset composition, and can learn consistently in a variety of settings. We find that KL-control and DQfD can be unstable in these settings.  

{Finally, in Figure \ref{fig:humanoid}, we  show that BEAR outperforms other considered prior methods in the challenging Humanoid-v2 environment as well, in two cases -- Medium-quality data and random data.}

\begin{wrapfigure}{r}{0.50\textwidth}
\vspace{-20pt}
        \includegraphics[width=0.49\linewidth]{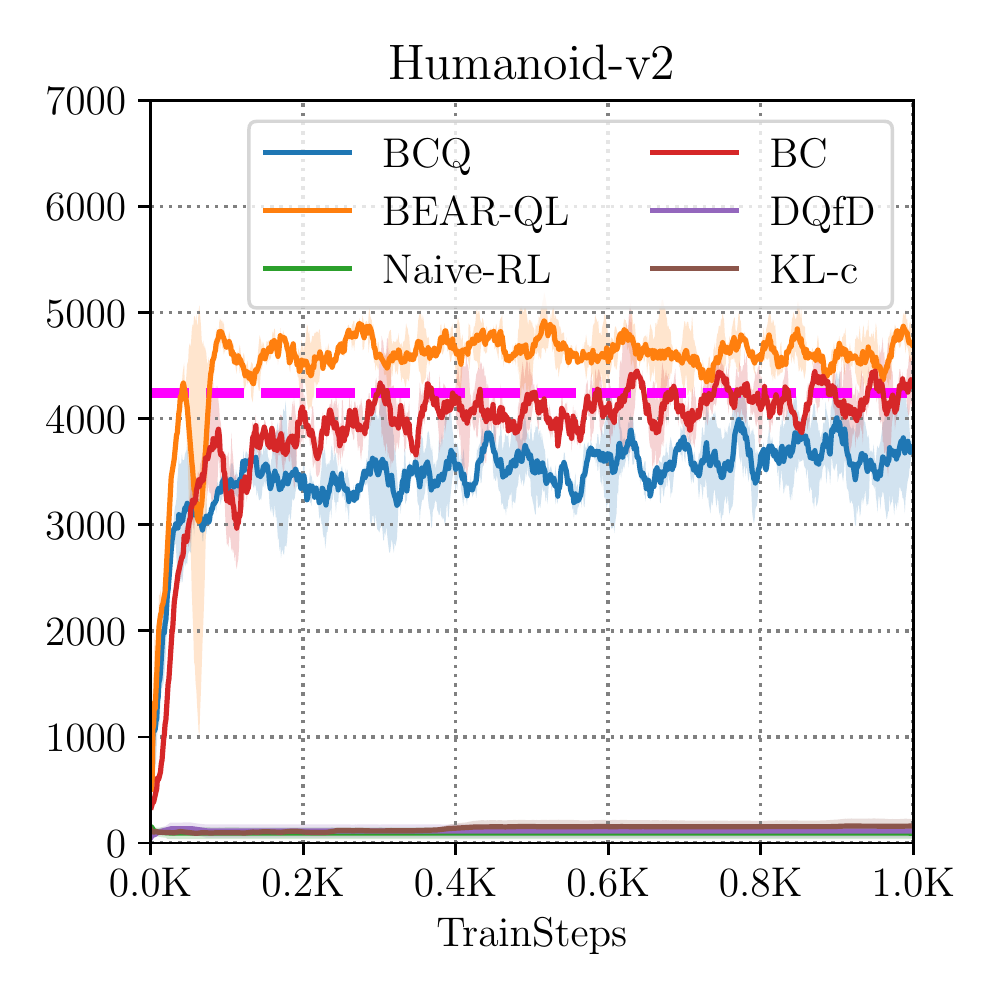}
       ~
        \includegraphics[width=0.48\linewidth]{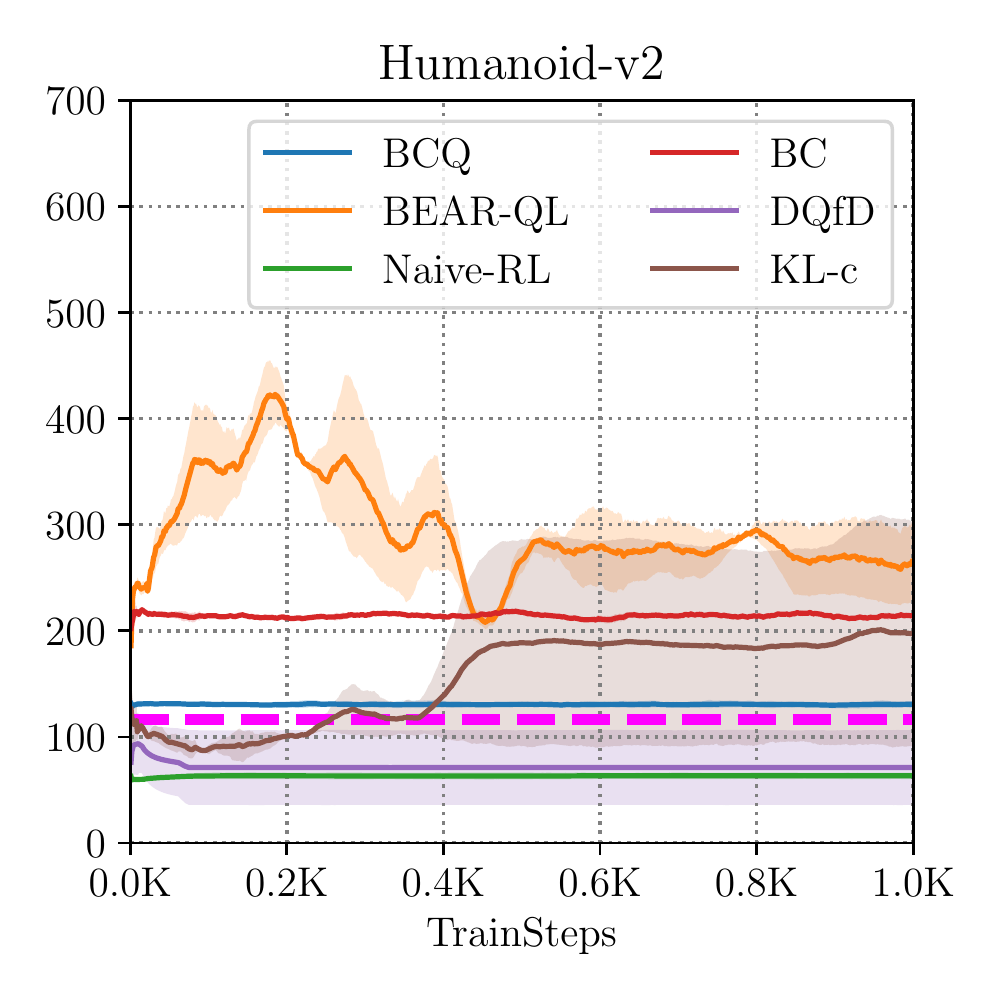}
      \caption{\footnotesize Performance of BEAR-QL, BCQ, Na\"ive RL and BC on medium-quality (left) and random (right) data in the Humanoid-v2 environment. Note that BEAR-QL outperforms prior methods.}
      \label{fig:humanoid}
\vspace{-10pt}
\end{wrapfigure}

\begin{figure*}[t!]
\vspace{-0.1in}
    \centering
    \begin{subfigure}[t]{0.23\textwidth}
        \centering
        \includegraphics[width=0.99\linewidth]{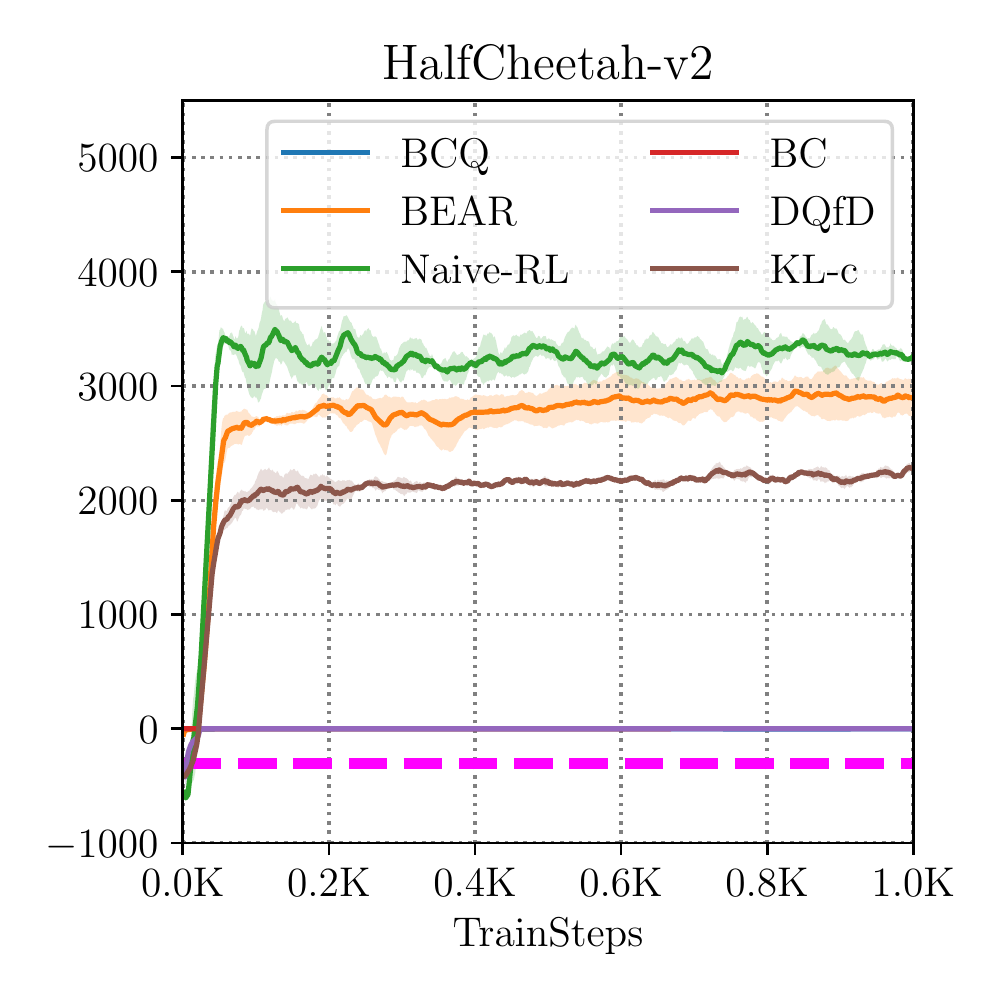}
        \includegraphics[width=0.99\linewidth]{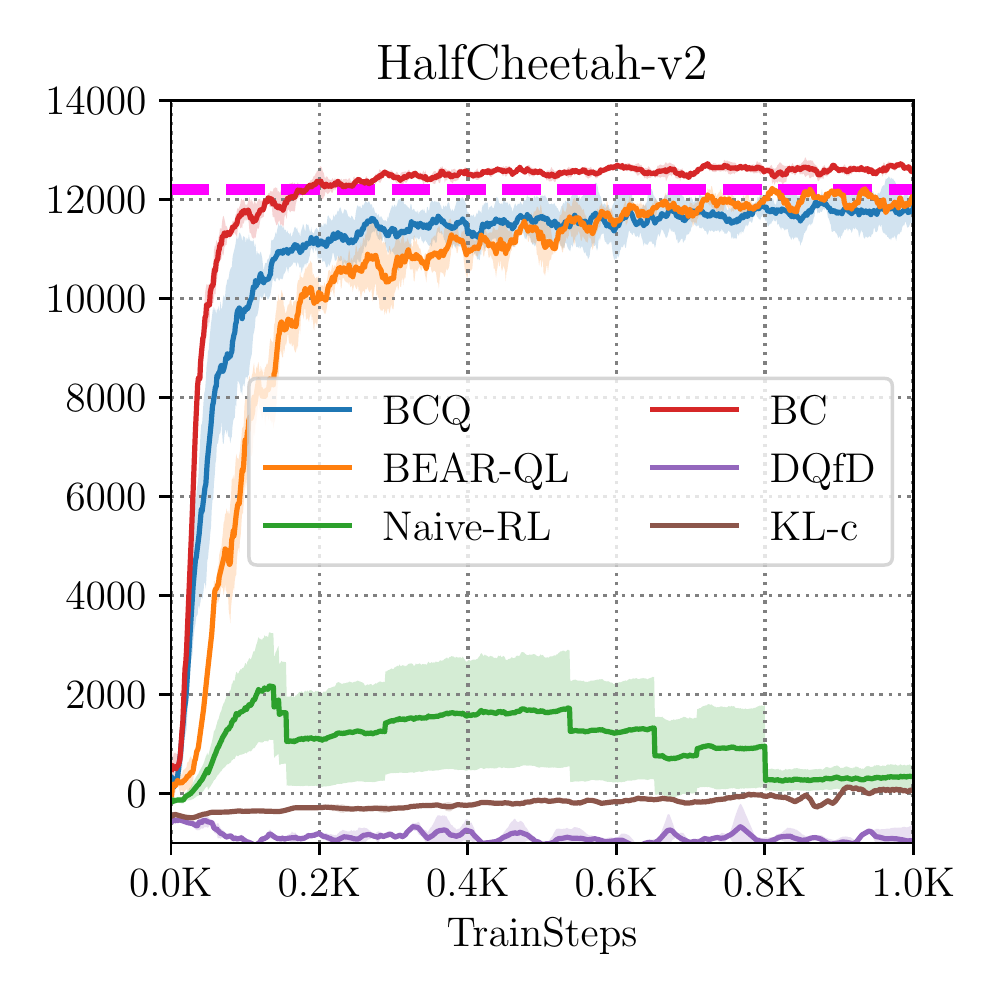}
    \end{subfigure}%
    ~ 
    \begin{subfigure}[t]{0.23\textwidth}
        \centering
        \includegraphics[width=0.99\linewidth]{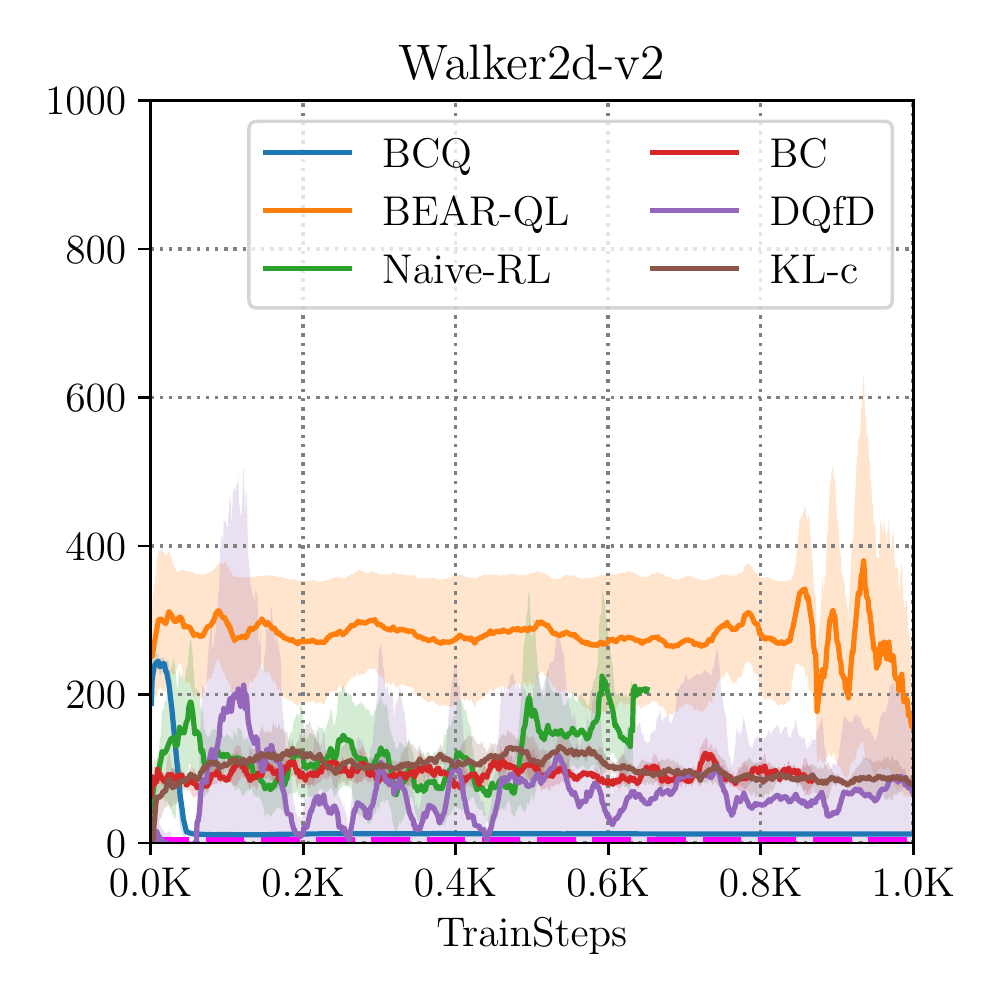}
        \includegraphics[width=0.99\linewidth]{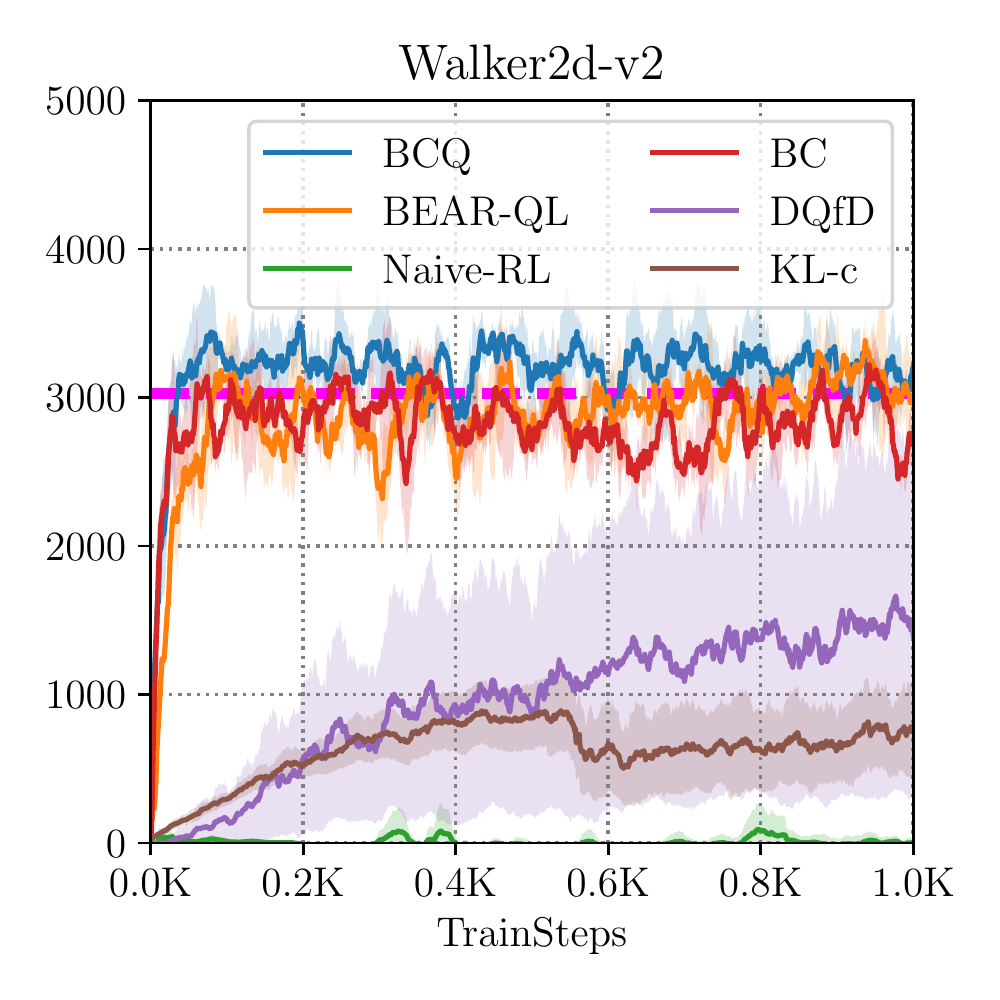}
    \end{subfigure}
    ~
    \begin{subfigure}[t]{0.23\textwidth}
        \centering
        \includegraphics[width=0.99\linewidth]{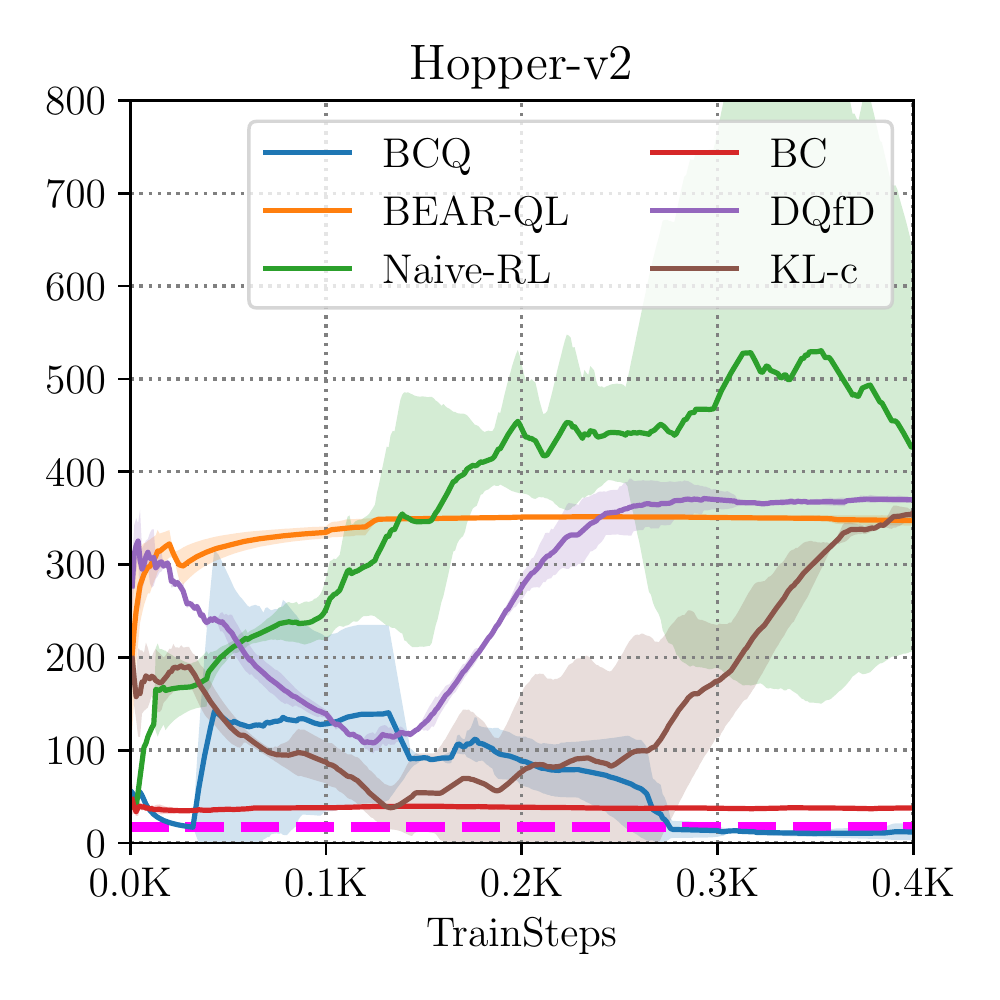}
        \includegraphics[width=0.99\linewidth]{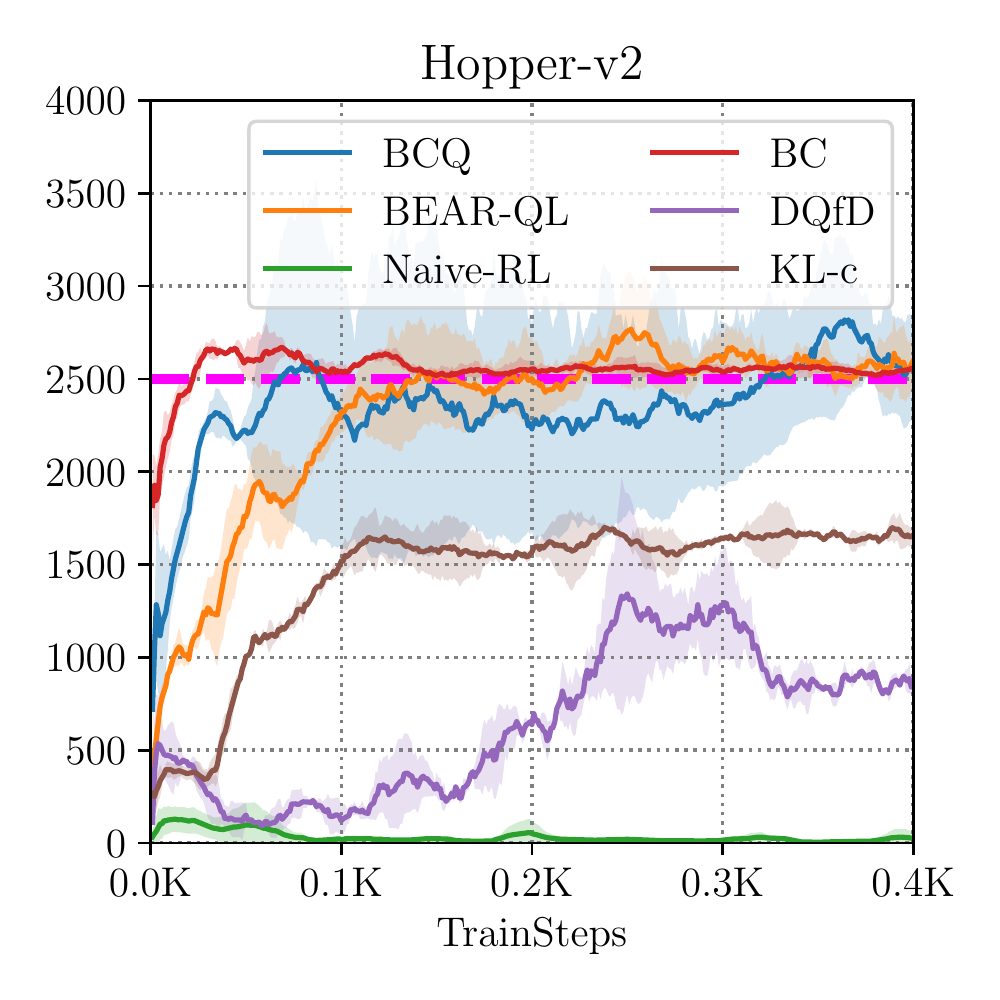}
    \end{subfigure}
    ~
    \begin{subfigure}[t]{0.23\textwidth}
        \centering
        \includegraphics[width=0.99\linewidth]{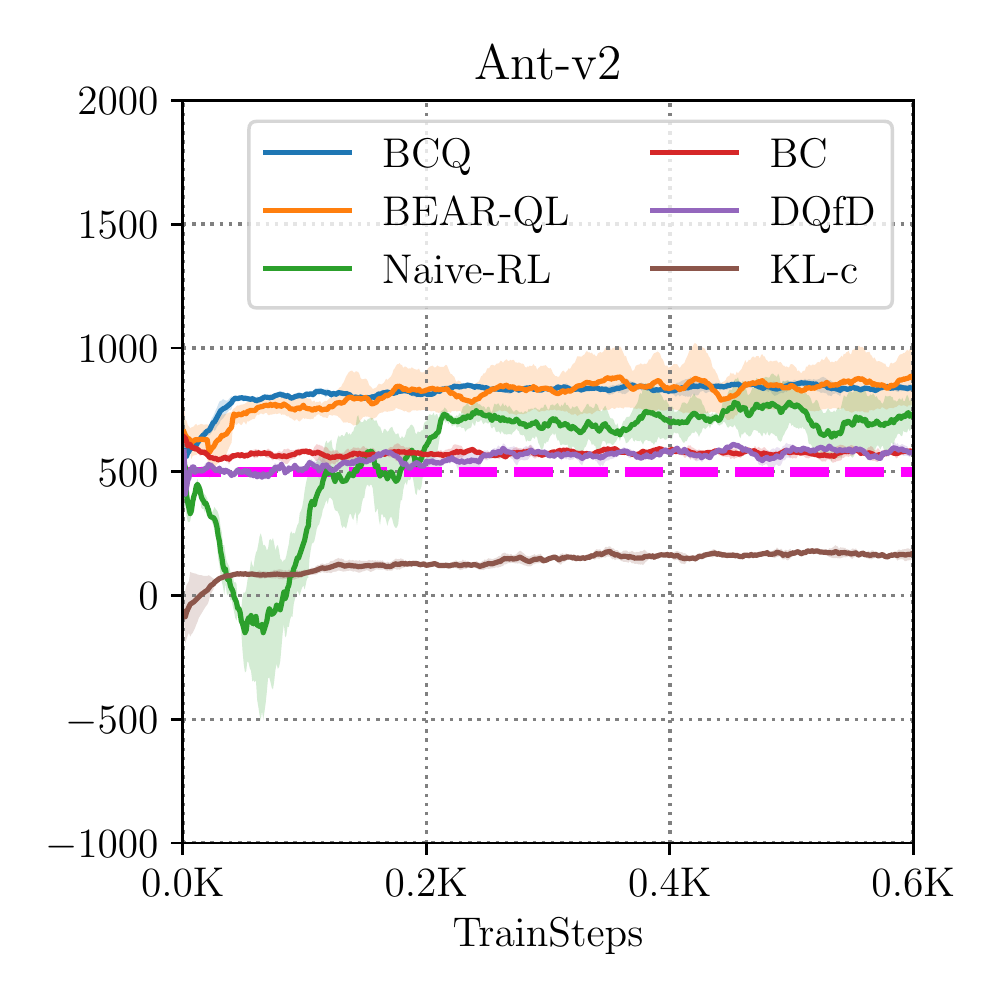}
        \includegraphics[width=0.99\linewidth]{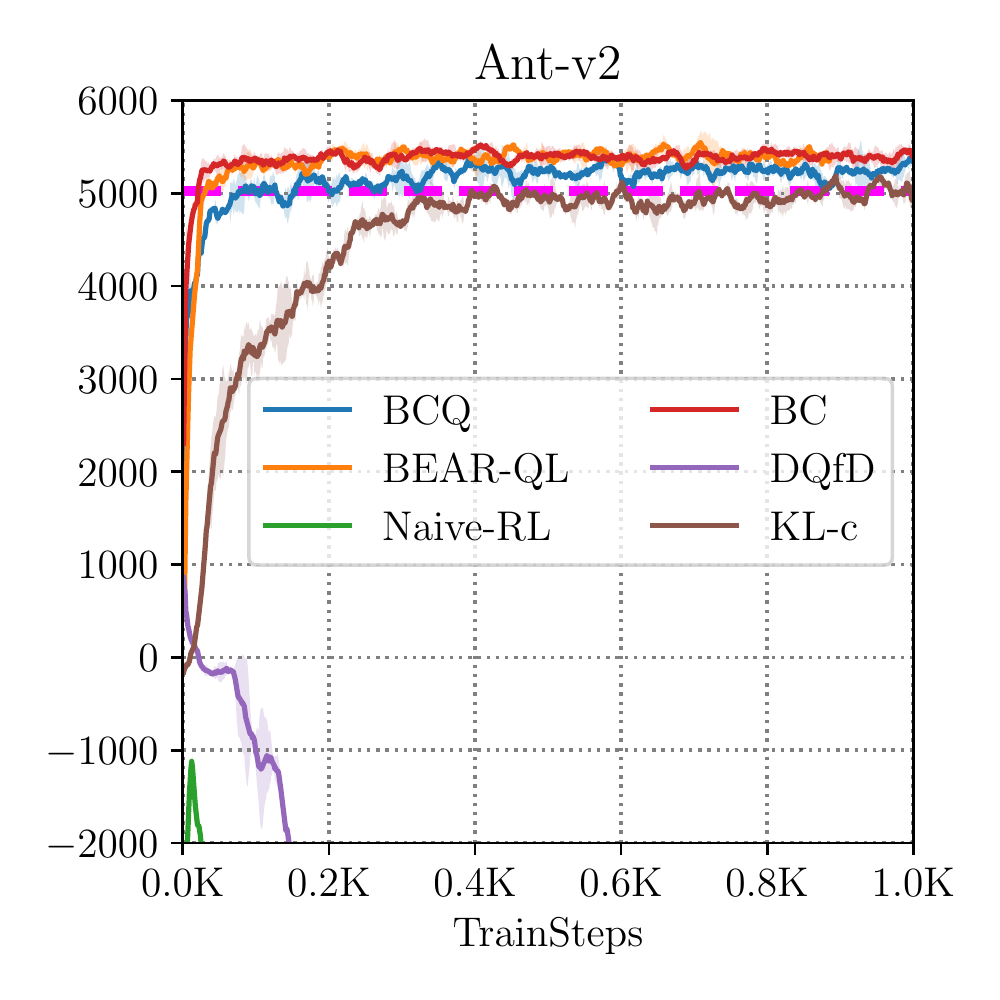}
    \end{subfigure}
    \caption{\footnotesize Average performance of BEAR-QL, BCQ, Na\"ive RL and BC on random data (top row) and optimal data (bottom row) over 5 seeds. BEAR-QL is the only algorithm capable of learning in both scenarios. Na\"{i}ve RL cannot handle optimal data, since it does not illustrate mistakes, and BCQ favors a behavioral cloning strategy (performs quite close to behaviour cloning in most cases), causing it to fail on random data. Average return over the training dataset is indicated by the dashed magenta line.}
    \label{fig:optimal_random}
    \vspace{-0.1in}
\end{figure*}

\vspace{-5pt}
\subsection{Analysis of BEAR-QL}
\label{subsec:ablations}
In this section, we aim to analyze different components of our method via an ablation study. Our first ablation studies the support constraint discussed in Section~\ref{sec:bear}, which uses MMD to measure support. We replace it with a more standard KL-divergence distribution constraint, which measures similarity in density. 
Our hypothesis is that this should provide a more conservative constraint, since matching distributions is not necessary for matching support. KL-divergence performs well in some cases, such as with optimal data, but as shown in Figure~\ref{fig:ablations}, it performs worse than MMD on medium-quality data. Even when KL-divergence is hand tuned fully, so as to prevent instability issues it still performs worse than a not-well tuned MMD constraint. We provide the results for this setting in the Appendix. We also vary the number of samples $n$ that are used to compute the MMD constraint. We find that smaller n ($\approx$ 4 or 5) gives better performance. Although the difference is not large, consistently better performance with 4 samples leans in favour of our hypothesis that an intermediate number of samples works well for support matching, and hence is less restrictive.



\section{Discussion and Future Work}
\vspace{-5pt}
The goal in our work was to study off-policy reinforcement learning with static datasets. We theoretically and empirically analyze how error propagates in off-policy RL due to the use of out-of-distribution actions for computing the target values in the Bellman backup. Our experiments suggest that this source of error is one of the primary issues afflicting off-policy RL: increasing the number of samples does not appear to mitigate the degradation issue (Figure~\ref{fig:divergence}), and training with na\"{i}ve RL on data from a random policy, where there are no out-of-distribution actions, shows much less degradation than training on data from more focused policies (Figure~\ref{fig:optimal_random}). Armed with this insight, we develop a method for mitigating the effect of out-of-distribution actions, which we call BEAR-QL. BEAR-QL constrains the backup to use actions that have non-negligible support under the data distribution, but without being overly conservative in constraining the learned policy. We observe experimentally that BEAR-QL achieves good performance across a range of tasks, and across a range of dataset compositions, learning well on random, medium-quality, and expert data.

\begin{wrapfigure}{r}{0.51\textwidth}
        \includegraphics[width=0.48\linewidth]{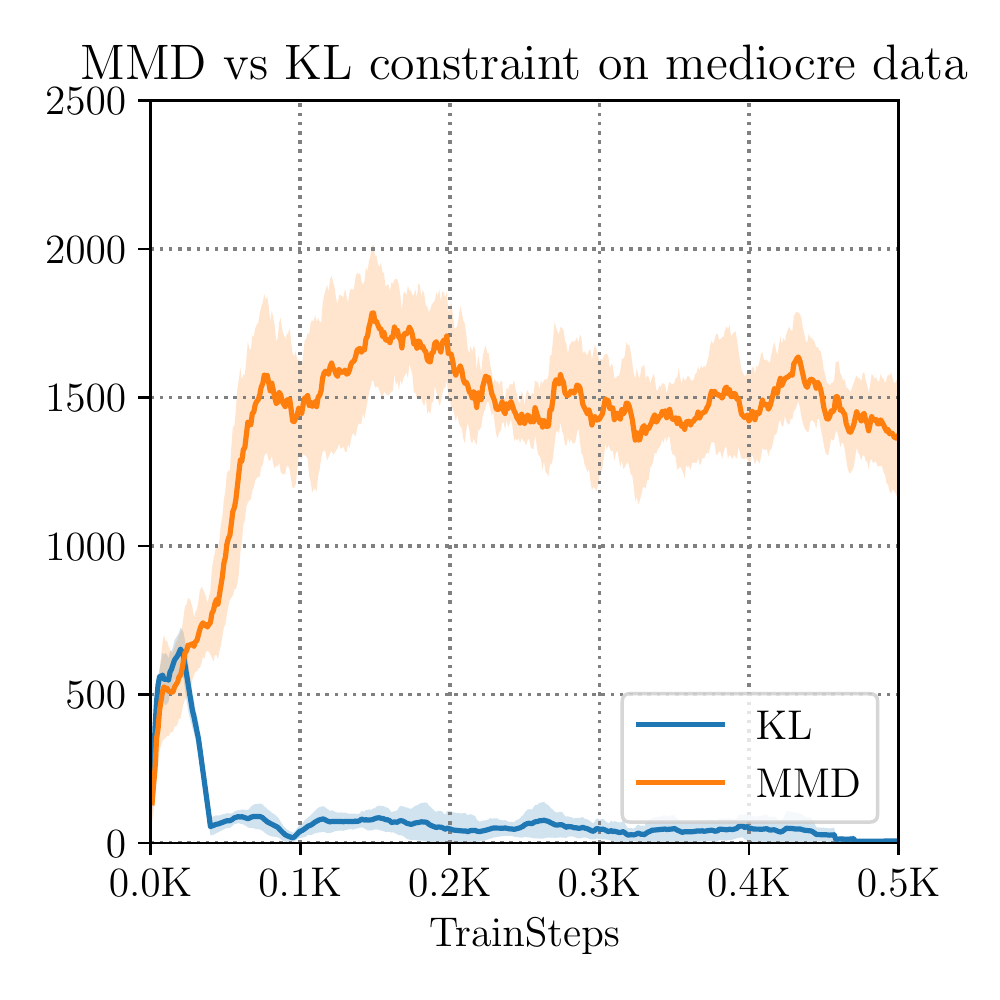}
       ~
        \includegraphics[width=0.48\linewidth]{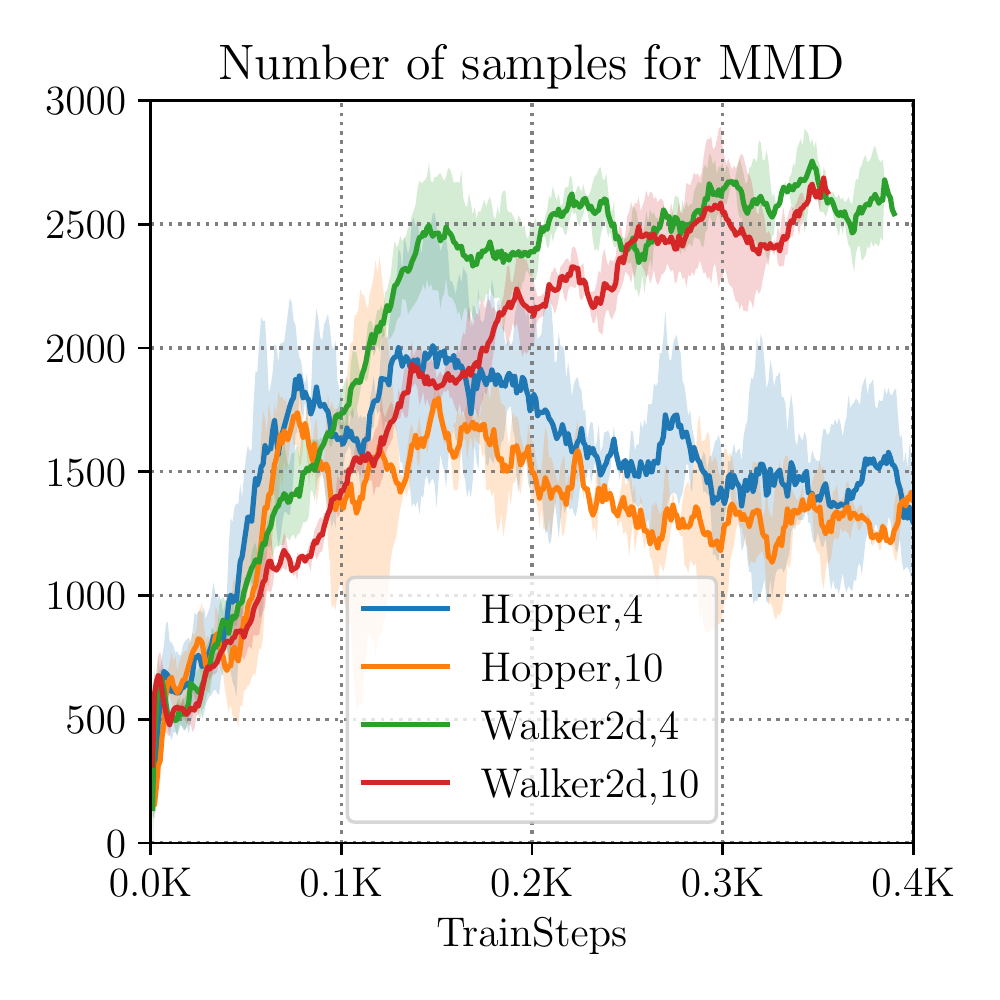}
      \caption{\footnotesize Average return (averaged Hopper-v2 and Walker2d-v2) as a function of train steps for ablation studies from Section~\ref{subsec:ablations}. (a) MMD constrained optimization is more stable and leads to better returns, (b) 4 sample MMD is more performant than 10.}
      \label{fig:ablations}
\vspace{-10pt}
\end{wrapfigure}

While BEAR-QL substantially stabilizes off-policy RL, we believe that this problem merits further study. One limitation of our current method is that, although the learned policies are more performant than those acquired with na\"{i}ve RL, performance sometimes still tends to degrade for long learning runs. An exciting direction for future work would be to develop an early stopping condition for RL, perhaps by generalizing the notion of validation error to reinforcement learning. {A limitation of approaches that perform constrained-action selection is that they can be overly conservative when compared to methods that constrain state-distributions directly, especially with datasets collected from mixtures of policies. We leave it to future work to design algorithms that can directly constrain state distributions. A theoretically robust method for support matching efficiently in high-dimensional continuous action spaces is a question for future research. Perhaps methods from outside RL, predominantly used in domain adaptation, such as using asymmetric f-divergences~\citep{wu19domain} can be used for support restriction.} Another promising future direction is to examine how well BEAR-QL can work on large-scale off-policy learning problems, of the sort that are likely to arise in domains such as robotics, autonomous driving, operations research, and commerce. If RL algorithms can learn effectively from large-scale off-policy datasets, reinforcement learning can become a truly data-driven discipline, benefiting from the same advantage in generalization that has been seen in recent years in supervised learning fields, where large datasets have enabled rapid progress in terms of accuracy and generalization~\cite{imagenet_cvpr09}.

\section*{Acknowledgements}
We thank Kristian Hartikainen for sharing implementations of RL algorithms and for help in debugging certain issues. We thank Matthew Soh for help in setting up environments. We thank Aurick Zhou, Chelsea Finn, Abhishek Gupta, Kelvin Xu and Rishabh Agarwal for informative discussions. We thank Ofir Nachum for comments on an earlier draft of this paper. We thank Google, NVIDIA, and Amazon for providing computational resources. This research was supported by Berkeley DeepDrive, JPMorgan Chase \& Co., NSF IIS-1651843 and IIS-1614653, the DARPA Assured Autonomy program, and ARL DCIST CRA W911NF-17-2-0181.

\bibliography{example_paper}

\begin{thebibliography}{39}
\providecommand{\natexlab}[1]{#1}
\providecommand{\url}[1]{\texttt{#1}}
\expandafter\ifx\csname urlstyle\endcsname\relax
  \providecommand{\doi}[1]{doi: #1}\else
  \providecommand{\doi}{doi: \begingroup \urlstyle{rm}\Url}\fi

\bibitem[Agarwal et~al.(2019)Agarwal, Schuurmans, and
  Norouzi]{agarwal19striving}
Rishabh Agarwal, Dale Schuurmans, and Mohammad Norouzi.
\newblock Striving for simplicity in off-policy deep reinforcement learning.
\newblock \emph{CoRR}, abs/1907.04543, 2019.
\newblock URL \url{http://arxiv.org/abs/1907.04543}.

\bibitem[{Antos} et~al.(2007){Antos}, {Szepesvari}, and {Munos}]{antos07value}
Andr\"as {Antos}, Csaa {Szepesvari}, and Remi {Munos}.
\newblock Value-iteration based fitted policy iteration: Learning with a single
  trajectory.
\newblock In \emph{2007 IEEE International Symposium on Approximate Dynamic
  Programming and Reinforcement Learning}, pages 330--337, April 2007.
\newblock \doi{10.1109/ADPRL.2007.368207}.

\bibitem[Antos et~al.(2008)Antos, Szepesv\'{a}ri, and Munos]{anots08fitted}
Andr\'{a}s Antos, Csaba Szepesv\'{a}ri, and R\'{e}mi Munos.
\newblock Fitted q-iteration in continuous action-space mdps.
\newblock In \emph{Advances in Neural Information Processing Systems 20}, pages
  9--16. Curran Associates, Inc., 2008.

\bibitem[Bennett et~al.(2007)Bennett, Lanning, et~al.]{bennett2007netflix}
James Bennett, Stan Lanning, et~al.
\newblock The netflix prize.
\newblock 2007.

\bibitem[Bertsekas and Tsitsiklis(1996)]{bertsekas1996ndp}
Dimitri~P Bertsekas and John~N Tsitsiklis.
\newblock \emph{Neuro-dynamic programming}.
\newblock Athena Scientific, 1996.

\bibitem[Byrd and Lipton()]{byrd19is}
Jonathon Byrd and Zachary Lipton.
\newblock What is the effect of importance weighting in deep learning?
\newblock In \emph{ICML 2019}.

\bibitem[de~Bruin et~al.(2015)de~Bruin, Kober, Tuyls, and
  Babuska]{deBruin2015importance}
Tim de~Bruin, Jens Kober, Karl Tuyls, and Robert Babuska.
\newblock The importance of experience replay database composition in deep
  reinforcement learning.
\newblock 01 2015.

\bibitem[Deng et~al.(2009)Deng, Dong, Socher, Li, Li, and
  Fei-Fei]{imagenet_cvpr09}
Jia Deng, Wei Dong, Richard~S. Socher, Li-Jia Li, Kai Li, and Li~Fei-Fei.
\newblock {ImageNet: A Large-Scale Hierarchical Image Database}.
\newblock In \emph{CVPR09}, 2009.

\bibitem[Devlin et~al.(2018)Devlin, Chang, Lee, and Toutanova]{devlin2018bert}
Jacob Devlin, Ming-Wei Chang, Kenton Lee, and Kristina Toutanova.
\newblock Bert: Pre-training of deep bidirectional transformers for language
  understanding.
\newblock \emph{arXiv preprint arXiv:1810.04805}, 2018.

\bibitem[Espeholt et~al.(2018)Espeholt, Soyer, Munos, Simonyan, Mnih, Ward,
  Doron, Firoiu, Harley, Dunning, et~al.]{impala2018}
Lasse Espeholt, Hubert Soyer, Remi Munos, Karen Simonyan, Volodymir Mnih, Tom
  Ward, Yotam Doron, Vlad Firoiu, Tim Harley, Iain Dunning, et~al.
\newblock Impala: Scalable distributed deep-rl with importance weighted
  actor-learner architectures.
\newblock In \emph{Proceedings of the International Conference on Machine
  Learning (ICML)}, 2018.

\bibitem[Farahmand et~al.(2010)Farahmand, Szepesv{\'a}ri, and
  Munos]{farahmand2010error}
Amir-massoud Farahmand, Csaba Szepesv{\'a}ri, and R{\'e}mi Munos.
\newblock Error propagation for approximate policy and value iteration.
\newblock In \emph{Advances in Neural Information Processing Systems}, pages
  568--576, 2010.

\bibitem[Fu et~al.(2019)Fu, Kumar, Soh, and Levine]{fu2019diagnosing}
Justin Fu, Aviral Kumar, Matthew Soh, and Sergey Levine.
\newblock Diagnosing bottlenecks in deep q-learning algorithms.
\newblock \emph{arXiv preprint arXiv:1902.10250}, 2019.

\bibitem[Fujimoto et~al.(2018{\natexlab{a}})Fujimoto, Meger, and
  Precup]{fujimoto2018off}
Scott Fujimoto, David Meger, and Doina Precup.
\newblock Off-policy deep reinforcement learning without exploration.
\newblock \emph{arXiv preprint arXiv:1812.02900}, 2018{\natexlab{a}}.

\bibitem[Fujimoto et~al.(2018{\natexlab{b}})Fujimoto, van Hoof, and
  Meger]{fujimoto18addressing}
Scott Fujimoto, Herke van Hoof, and David Meger.
\newblock Addressing function approximation error in actor-critic methods.
\newblock In Jennifer Dy and Andreas Krause, editors, \emph{Proceedings of the
  35th International Conference on Machine Learning}, volume~80 of
  \emph{Proceedings of Machine Learning Research}, pages 1587--1596. PMLR,
  2018{\natexlab{b}}.

\bibitem[Gao et~al.(2018)Gao, Xu, Lin, Yu, Levine, and
  Darrell]{DBLP:conf/iclr/GaoXLYLD18}
Yang Gao, Huazhe Xu, Ji~Lin, Fisher Yu, Sergey Levine, and Trevor Darrell.
\newblock Reinforcement learning from imperfect demonstrations.
\newblock In \emph{{ICLR} (Workshop)}. OpenReview.net, 2018.

\bibitem[Gelada and Bellemare(2019)]{gelada2019off}
Carles Gelada and Marc~G. Bellemare.
\newblock Off-policy deep reinforcement learning by bootstrapping the covariate
  shift.
\newblock \emph{CoRR}, abs/1901.09455, 2019.

\bibitem[Grau-Moya et~al.(2019)Grau-Moya, Leibfried, and
  Vrancx]{grau-moya2018soft}
Jordi Grau-Moya, Felix Leibfried, and Peter Vrancx.
\newblock Soft q-learning with mutual-information regularization.
\newblock In \emph{International Conference on Learning Representations}, 2019.
\newblock URL \url{https://openreview.net/forum?id=HyEtjoCqFX}.

\bibitem[Gretton et~al.(2012)Gretton, Borgwardt, Rasch, Sch\"{o}lkopf, and
  Smola]{gretton2012kernel}
Arthur Gretton, Karsten~M. Borgwardt, Malte~J. Rasch, Bernhard Sch\"{o}lkopf,
  and Alexander Smola.
\newblock A kernel two-sample test.
\newblock \emph{J. Mach. Learn. Res.}, 13:\penalty0 723--773, March 2012.
\newblock ISSN 1532-4435.
\newblock URL \url{http://dl.acm.org/citation.cfm?id=2188385.2188410}.

\bibitem[Haarnoja et~al.(2018)Haarnoja, Zhou, Abbeel, and
  Levine]{haarnoja2018sac}
Tuomas Haarnoja, Aurick Zhou, Pieter Abbeel, and Sergey Levine.
\newblock Soft actor-critic: Off-policy maximum entropy deep reinforcement
  learning with a stochastic actor.
\newblock \emph{arXiv preprint arXiv:1801.01290}, 2018.

\bibitem[He et~al.(2016)He, Zhang, Ren, and Sun]{he2016resnet}
Kaiming He, Xiangyu Zhang, Shaoqing Ren, and Jian Sun.
\newblock Deep residual learning for image recognition.
\newblock \emph{2016 IEEE Conference on Computer Vision and Pattern Recognition
  (CVPR)}, pages 770--778, 2016.

\bibitem[Hester et~al.(2018)Hester, Vecerik, Pietquin, Lanctot, Schaul, Piot,
  Horgan, Quan, Sendonaris, Osband, et~al.]{hester2018dqfd}
Todd Hester, Matej Vecerik, Olivier Pietquin, Marc Lanctot, Tom Schaul, Bilal
  Piot, Dan Horgan, John Quan, Andrew Sendonaris, Ian Osband, et~al.
\newblock Deep q-learning from demonstrations.
\newblock In \emph{Thirty-Second AAAI Conference on Artificial Intelligence},
  2018.

\bibitem[Jaques et~al.(2019)Jaques, Ghandeharioun, Shen, Ferguson, Lapedriza,
  Jones, Gu, and Picard]{jacques19way}
Natasha Jaques, Asma Ghandeharioun, Judy~Hanwen Shen, Craig Ferguson,
  {\`{A}}gata Lapedriza, Noah Jones, Shixiang Gu, and Rosalind~W. Picard.
\newblock Way off-policy batch deep reinforcement learning of implicit human
  preferences in dialog.
\newblock \emph{CoRR}, abs/1907.00456, 2019.
\newblock URL \url{http://arxiv.org/abs/1907.00456}.

\bibitem[Kakade and Langford(2002)]{kakade2002cpi}
Sham Kakade and John Langford.
\newblock Approximately optimal approximate reinforcement learning.
\newblock In \emph{Proceedings of the Nineteenth International Conference on
  Machine Learning}, pages 267--274. Morgan Kaufmann Publishers Inc., 2002.

\bibitem[Kalashnikov et~al.(2018)Kalashnikov, Irpan, Pastor, Ibarz, Herzog,
  Jang, Quillen, Holly, Kalakrishnan, Vanhoucke, and
  Levine]{kalashnikov18qtopt}
Dmitry Kalashnikov, Alex Irpan, Peter Pastor, Julian Ibarz, Alexander Herzog,
  Eric Jang, Deirdre Quillen, Ethan Holly, Mrinal Kalakrishnan, Vincent
  Vanhoucke, and Sergey Levine.
\newblock Scalable deep reinforcement learning for vision-based robotic
  manipulation.
\newblock In \emph{Proceedings of The 2nd Conference on Robot Learning},
  volume~87 of \emph{Proceedings of Machine Learning Research}, pages 651--673.
  PMLR, 2018.

\bibitem[Laroche et~al.(2019)Laroche, Trichelair, and
  Des~Combes]{laroche2019spibb}
Romain Laroche, Paul Trichelair, and Remi~Tachet Des~Combes.
\newblock Safe policy improvement with baseline bootstrapping.
\newblock In \emph{International Conference on Machine Learning (ICML)}, 2019.

\bibitem[Levine(2018)]{levine2018rlasinference}
Sergey Levine.
\newblock Reinforcement learning and control as probabilistic inference:
  Tutorial and review.
\newblock \emph{CoRR}, abs/1805.00909, 2018.
\newblock URL \url{http://arxiv.org/abs/1805.00909}.

\bibitem[Mahmood et~al.(2015)Mahmood, Yu, White, and
  Sutton]{mahmood2015emphatic}
A~Rupam Mahmood, Huizhen Yu, Martha White, and Richard~S Sutton.
\newblock Emphatic temporal-difference learning.
\newblock \emph{arXiv preprint arXiv:1507.01569}, 2015.

\bibitem[Munos(2003)]{munos2003errorapi}
R{\'e}mi Munos.
\newblock Error bounds for approximate policy iteration.
\newblock In \emph{Proceedings of the Twentieth International Conference on
  International Conference on Machine Learning}, pages 560--567. AAAI Press,
  2003.

\bibitem[Munos(2005)]{munos2005erroravi}
R{\'e}mi Munos.
\newblock Error bounds for approximate value iteration.
\newblock In \emph{Proceedings of the National Conference on Artificial
  Intelligence}, 2005.

\bibitem[Munos et~al.(2016)Munos, Stepleton, Harutyunyan, and
  Bellemare]{munos2016safe}
R{\'e}mi Munos, Tom Stepleton, Anna Harutyunyan, and Marc Bellemare.
\newblock Safe and efficient off-policy reinforcement learning.
\newblock In \emph{Advances in Neural Information Processing Systems}, pages
  1054--1062, 2016.

\bibitem[Precup et~al.(2001)Precup, Sutton, and Dasgupta]{precup2001offpol}
Doina Precup, Richard~S. Sutton, and Sanjoy Dasgupta.
\newblock {Off-policy temporal-difference learning with function
  approximation}.
\newblock In \emph{International Conference on Machine Learning (ICML)}, 2001.

\bibitem[Schaal(1999)]{Schaal99isimitation}
Stefan Schaal.
\newblock Is imitation learning the route to humanoid robots?, 1999.

\bibitem[Schaul et~al.(2016)Schaul, Quan, Antonoglou, and
  Silver]{Schaul2016PrioritizedER}
Tom Schaul, John Quan, Ioannis Antonoglou, and David Silver.
\newblock Prioritized experience replay.
\newblock \emph{CoRR}, abs/1511.05952, 2016.

\bibitem[Scherrer et~al.(2015)Scherrer, Ghavamzadeh, Gabillon, Lesner, and
  Geist]{bruno2015approximate}
Bruno Scherrer, Mohammad Ghavamzadeh, Victor Gabillon, Boris Lesner, and
  Matthieu Geist.
\newblock Approximate modified policy iteration and its application to the game
  of tetris.
\newblock \emph{Journal of Machine Learning Research}, 16:\penalty0 1629--1676,
  2015.
\newblock URL \url{http://jmlr.org/papers/v16/scherrer15a.html}.

\bibitem[Schulman et~al.(2015)Schulman, Levine, Abbeel, Jordan, and
  Moritz]{schulman2015trpo}
John Schulman, Sergey Levine, Pieter Abbeel, Michael Jordan, and Philipp
  Moritz.
\newblock Trust region policy optimization.
\newblock In Francis Bach and David Blei, editors, \emph{Proceedings of the
  32nd International Conference on Machine Learning}, volume~37 of
  \emph{Proceedings of Machine Learning Research}, pages 1889--1897, Lille,
  France, 07--09 Jul 2015. PMLR.

\bibitem[Sutton and Barto(2018)]{suttonrlbook}
Richard~S Sutton and Andrew~G Barto.
\newblock \emph{Reinforcement learning: An introduction}.
\newblock Second edition, 2018.

\bibitem[Todorov et~al.(2012)Todorov, Erez, and Tassa]{mujoco}
Emanuel Todorov, Tom Erez, and Yuval Tassa.
\newblock {MuJoCo}: A physics engine for model-based control.
\newblock In \emph{IROS}, pages 5026--5033, 2012.

\bibitem[Wu et~al.()Wu, Winston, Kaushik, and Lipton]{wu19domain}
Yifan Wu, Ezra Winston, Divyansh Kaushik, and Zachary Lipton.
\newblock Domain adaptation with asymmetrically-relaxed distribution alignment.
\newblock In \emph{ICML 2019}.

\bibitem[Yu et~al.(2018)Yu, Xian, Chen, Liu, Liao, Madhavan, and
  Darrell]{yu2018bdd}
Fisher Yu, Wenqi Xian, Yingying Chen, Fangchen Liu, Mike Liao, Vashisht
  Madhavan, and Trevor Darrell.
\newblock {BDD100K:} {A} diverse driving video database with scalable
  annotation tooling.
\newblock \emph{CoRR}, abs/1805.04687, 2018.
\newblock URL \url{http://arxiv.org/abs/1805.04687}.

\end{thebibliography}
\bibliographystyle{plainnat}

\newpage
\appendix
\part*{Appendices}

\section{Distribution-Constrained Backup Operator}
\label{app:constrained_backup}
In this section, we analyze properties of the constrained Bellman backup operator, defined as:
\[ \TPi Q(s, a) \defeq \expec \big[ R(s, a) + \gamma \max_{\pi \in \Pi} \expec_{\trans(s' | s, a)}\left[V(s') \right] \big] \]
where
\[V(s) \defeq \max_{\pi \in \Pi} \expec_{\pi}[Q(s, a)].\]

Such an operator can be reduced to a standard Bellman backup in a modified MDP. We can construct an MDP $M'$ from the original MDP $M$ as follows:

\begin{itemize}
    \item The state space, discount, and initial state distributions remain unchanged from $M$.
    \item We define a new action set $\mathcal{A}' = \Pi$ to be the choice of policy $\pi$ to execute.  
    \item We define the new transition distribution $p'$ as taking one step under the chosen policy $\pi$ to execute and one step under the original dynamics $p$: $p'(s'|s, \pi) = E_{\pi}[p(s'|s,a)]$.
    \item Q-values in this new MDP, $Q^\Pi(s, \pi)$ would, in words, correspond to executing policy $\pi$ for one step and executing the policy which maximizes the future discounted value function in the original MDP $M$ thereafter.   
\end{itemize}

Under this redefinition, the Bellman operator $\TPi$ is mathematically the same operation as the Bellman operator under $M'$. Thus, standard results from MDP theory carry over - i.e. the existence of a fixed point and convergence of repeated application of $\TPi$ to said fixed point.

\section{Error Propagation}
\label{app:error_prop}
In this section, we provide proofs for  Theorem~\ref{thm:avi_bound} and Theorem~\ref{thm:conc_coeff_bound}.
\begin{theorem}
\label{thm:avi_bound_proof}
Suppose we run approximate distribution-constrained value iteration with a set constrained backup $\TPi$. Assume that $\delta(s,a) \ge \max_k |Q_k(s,a) - \TPi Q_{k-1}(s,a)|$ bounds the Bellman error. Then,
\[\lim_{k \to \infty} \expec_{\rhoinit}[|V_k(s) - V^*(s)|] \le 
\frac{\gamma}{(1-\gamma)^2}\left[ C(\Pi)\expec_\mu[\max_{\pi \in \Pi} \expec_{\pi}[\projerr(s,a)]] + \frac{1-\gamma}{\gamma}\alpha(\Pi) \right]
\]
\end{theorem}
\begin{proof}
We first begin by introducing $\VPi$, the fixed point of $\TPi$. By the triangle inequality, we have:
\begin{align*}
\expec_{\rhoinit}[|V_k(s) - V^*(s)|] &= \expec_{\rhoinit}[|V_k(s,a) - \VPi(s) + \VPi(s) - V^*(s)|]\\
&\le \underbrace{\expec_{\rhoinit}[|V_k(s) - \VPi(s)|]}_{L_1} + \underbrace{\expec_{\rhoinit}[|\VPi(s) - V^*(s)|]}_{L_2}
\end{align*}


First, we note that $\max_\pi \expec_{\pi}[\delta(s,a)]$ provides an upper bound on the value error:
\begin{align*}
|V_k(s) - \TPi V_{k-1}(s)| &= |\max_\pi \expec_{\pi}[Q_k(s,a)] - \max_\pi \expec_{\pi}[\Tpi Q_{k-1}(s,a)]| \\
&\le \max_\pi\expec_{\pi}[| Q_k(s,a) - \Tpi Q_{k-1}(s,a)|]\\
&\le \max_\pi\expec_{\pi}[\projerr(s,a)]
\end{align*}

We can bound $L_1$ with
\[
L_1 \le \frac{2\gamma}{(1-\gamma)^2}[C(\Pi)]\expec_\mu[\max_{\pi \in \Pi} \expec_{\pi}[\delta(s, a)]]
\]
by direct modification of the proof of Theorem 3 of \citet{farahmand2010error} or Theorem 1 of~\citet{munos2005erroravi} with $k=1$ ($p=1$), but replacing $V^*$  with $\VPi$ and $\backup$ with $\TPi$, as $\TPi$ is a contraction and $\VPi$ is its fixed point.
An alternative proof involves viewing $\TPi$ as a backup under a modified MDP (see Appendix~\ref{app:constrained_backup}), and directly apply Theorem 1 of~\citet{munos2005erroravi} under this modified MDP. A similar bound also holds true for value iteration with the $\TPi$ operator which can be analysed on similar lines as the above proof and \citet{munos2005erroravi}.

To bound $L_2$, we provide a simple $\linfnorm$-norm bound, although we could in principle apply techniques used to bound $L_1$ to get a tighter distribution-based bound.
\begin{align*}
\norminf{\VPi - V^*} &= \norminf{ \TPi\VPi - \backup V^*} \\ 
&\le \norminf{\TPi\VPi - \TPi V^*} + \norminf{\TPi\VPi - \backup V^*} \\ 
&\le \gamma \norminf{\VPi - V^*} + \alpha(\Pi)
\end{align*}
Thus, we have $\norminf{\VPi - V^*} \le \frac{\alpha}{1-\gamma}$. Because the maximum is greater than the expectation, $L_2 = \expec_{\rhoinit, \pi}[|\VPi(s) - V^*(s)|] \le \norminf{\VPi - V^*}$.

Adding $L_1$ and $L_2$ completes the proof.
\end{proof}

\begin{theorem}
\label{thm:conc_coeff_proof}
Assume the data distribution $\mu$ is generated by a behavior policy $\beta$, such that $\mu(s,a) = \mu_\beta(s,a)$. Let $\mu(s)$ be the marginal state distribution under the data distribution. Let us define $\Pieps = \{ \pi ~|~ \pi( a | s) = 0 \text{ whenever } \beta( a | s) < \epsilon \}$. Then, there exists a concentrability coefficient $C(\Pieps)$ is bounded as:
\[
C(\Pi_\epsilon) \leq C(\beta) \cdot \Big(1 + \frac{\gamma}{(1 - \gamma) f(\epsilon)} (1 - \epsilon)\Big)
\]
where $f(\epsilon) \defeq \min_{s \in \mathcal{S}, \mu_\Pi(s) > 0} [\mu(s)]$.
\begin{proof}
For notational clarity, we refer to $\Pi_\epsilon$ as $\Pi$ in this proof. The term $\mu_\Pi$ is the highest discounted marginal state distribution starting from the initial state distribution $\rho$ and following policies $\pi \in \Pi$. Formally, it is defined as:
$$ \mu_{\Pi} \defeq \max_{\{\pi_i\}_i :~ \forall~i,~\pi_i \in \Pi} (1 - \gamma) \sum_{m=1}^{\infty} m \gamma^{m-1} \rhoinit P^{\pi_1} \cdots P^{\pi_m} $$


Now, we begin the proof of the theorem. We first note, from the definition of $\Pi$, $\forall~s \in \mathcal{S}~\forall~\pi \in \Pi, \pi(a|s) > 0 \implies \beta(a|s) > \epsilon$. This suggests a bound on the total variation distance between $\beta$ and any $\pi \in \Pi$ for all $s \in \mathcal{S}$, $D_{TV}(\beta(\cdot|s) ||\pi(\cdot|s)) \leq 1 - \epsilon$. This means that the marginal state distributions of $\beta$ and $\Pi$, are bounded in total variation distance by: $D_{TV}(\mu_{\beta}|| \mu_{\Pi}) \leq \frac{\gamma}{1 - \gamma} (1 - \epsilon)$, where $\mu_{\Pi}$ is the marginal state distribution as defined above. This can be derived from~\citet{schulman2015trpo}, Appendix B, which bounds the difference in returns of two policies by showing the state marginals between two policies are bounded if their total variation distance is bounded.

Further, the definition of the set of policies $\Pi$ implies that $\forall~s \in \mathcal{S}, \mu_{\Pi}(s) > 0 \implies \mu_{\beta}(s) \geq f(\epsilon)$, where $f(\epsilon) > 0$ is a constant that depends on $\epsilon$ and captures the minimum visitation probability of a state $s \in \mathcal{S}$ when rollouts are executed from the initial state distribution $\rho$ while executing the behaviour policy $\beta(a|s)$, under the constraint that only actions with $\beta(a|s) \geq \epsilon$ are selected for execution in the environment. Combining it with the total variation divergence bound, $\max_s ||\mu_{\beta}(s) - \mu_{\Pi}(s)|| \leq \frac{\gamma}{1 - \gamma} (1 - \epsilon)$, we get that 
$$\sup_{s \in \mathcal{S}} \frac{\mu_{\Pi}(s)}{\mu_{\beta}(s)} \leq 1 + \frac{\gamma}{(1 - \gamma) f(\epsilon)} (1 - \epsilon)$$

We know that, $C(\Pi) \defeq (1-\gamma)^2\sum_{k=1}^\infty k\gamma^{k-1}c(k)$ is the ratio of the marginal state visitation distribution under the policy iterates when performing backups using the distribution-constrained operator and the data distribution $\mu = \mu_\beta$. Therefore, $$\frac{C(\Pi_\epsilon)}{C(\beta)} \defeq \sup_{s \in \mathcal{S}} \frac{\mu_\Pi(s)}{\mu_\beta(s)} \leq 1 + \frac{\gamma}{(1 - \gamma) f(\epsilon)} (1 - \epsilon) $$
\end{proof}
\end{theorem}

\section{Additional Details Regarding BEAR-QL}
\label{app:bearql-more}

In this appendix, we address several remaining points regarding the support matching formulation of BEAR-QL, and further discuss its connections to prior work.

\subsection{Why can we choose actions from $\Pieps$, the support of the training distribution, and need not restrict action selection to the policy distribution?}
In Section~\ref{sec:dist_constrained}, we designed a new distribution-constrained backup and analyzed its properties from an error propagation perspective. Theorems~\ref{thm:avi_bound} and~\ref{thm:conc_coeff_bound} tell us that, if the maximum projection error on all actions within the support of the train distribution is bounded, then the worst-case error incurred is also bounded. That is, we have a bound on \mbox{$\max_{\pi \in \Pieps} \expec_\pi [\delta_k(s, a)]$}. In this section, we provide an intuitive explanation for why action distributions that are very different from the training policy distributions, but still lie in the support of the train distribution, can be chosen without incurring large error. In practice, we use powerful function approximators for Q-learning, such as deep neural networks. That is, $\delta_k(s, a)$ is the Bellman error for one iteration of Q-iteration/Q-learning, which can essentially be viewed as a supervised regression problem with a very expressive function class. In this scenario, we expect a bounded error on the entire support of the training distribution, and we therefore expect approximation error to depend less on the specific density of a datapoint under the data distribution, and more on whether or not that datapoint is within the support of the data distribution. I.e., any point that is within the support would have a comparatively low error, due to the expressivity of the function approximator.

Another justification is that, a different version of the Bellman error objective renormalizes the action-distributions to the uniform distribution by applying an inverse behavior policy density weighting. For example, \cite{anots08fitted, antos07value} use this variant of Bellman error: 
$$Q_{k+1}= {\operatorname{argmin}}_{Q} \sum_{i=1, a_i \sim \beta(\cdot|s_i)}^{N} \frac{1}{\beta\left(a_{i} | s_{i}\right)}\left(Q\left(s_{i}, a_{i}\right)-\left[R{(s, a)}+\gamma \max _{a' \in \mathcal{A}} Q_{k}\left(s_{i+1}, a'\right)\right]\right)^{2}$$
This implies that this form of Bellman error mainly depends upon the support of the behaviour policy $\beta$ (i.e. the set of action samples sampled from $\beta$ with a high-enough probability which we formally refer to as $\beta(a|s) \geq \epsilon$ in the main text). In a scenario when this form of Bellman error is being minimized, $\delta_k(s, a)$ is defined as
$$\delta_k(s, a) = \frac{1}{\beta(a|s)} \left| Q_k(s, a) - \backup Q_{k-1}(s, a)\right| $$ 
The overall error, hence, incurred due to error propagation is expected to be insensitive to distribution change, provided the support of the distribution doesn't change. Therefore, all policies $\pi \in \Pieps$ incur the same amount of propagated error ($|V_k - V_{\Pi}|$) whereas different amount of suoptimality biases -- suggesting the existence of a different policy in $\Pieps$ which propagates the same amount of error while having a lower suboptimality bias. However, in practice, it has been observed that using the inverse density weighting under the behaviour policy doesn't lead to substantially better performance for vanilla RL (not in the setting with purely off-policy, static datasets), so we use the unmodified Bellman error objective.

Both of these justifications indicate that bounded $\delta_k(s, a)$ is reasonable to expect under in-support action distributions.

\subsection{Details on connection between BEAR-QL and distribution-constrained backups}
Distribution-constrained backups perform maximization over a set of policies $\Pieps$ which is defined as the set of policies that share the support with the behaviour policy. In the BEAR-QL algorithm, $\pi_\phi$ is maximized towards maximizing the expected Q-value for each state under the action distribution defined by it, while staying in-support (through the MMD constraint). The maximization step biases $\pi_\phi$ towards the in-support actions which maximize the current Q-value. By sampling multiple Dirac-delta action distributions -  $\delta_{a_i}$ - and then performing an explicit maximum over them for computing the target is a stochastic approximation to the distribution-constrained operator. What is the importance of training the actor to maximize the expected Q-value? We found empirically that this step is important as without this maximization step and high-dimensional action spaces, it is likely to require many more samples (exponentially more, due to curse of dimensionality) to get the correct action that maximizes the target value while being in-support. This is hard and unlikely, and in some experiments we tried with this variant, we found it to lead to suboptimal solutions. At evaluation time, we use the Q-function as the actor. The same process is followed. Dirac-delta action distribution candidates $\delta_{a_i}$ are sampled, and then the action $a_i$ that is gives the empirical maximum over the Q-function values is the action that is executed in the environment.
 
\subsection{How effective is the $\operatorname{MMD}$ constraint in constraining supports of distributions? }
\label{app:mmd}
In Section \ref{sec:bear}, we argued in favour of the usage of the sampled $\operatorname{MMD}$ distance between distributions to search for a policy that is supported on the same support as the train distribution. Revisiting the argument, in this section, we argue, via numerical simulations, the effectiveness of the $\operatorname{MMD}$ distance between two probability distributions in constraining the support of the distribution being learned, without constraining the distribution density function too much. While, MMD distance computed exactly between two distribution functions will match distributions exactly and that explains its applicability in 2-sample tests, however, with a limited number of samples, we empirically find that the values of the $\mmd$ distance computed using samples from two $d$-dimensional Gaussian distributions with diagonal covariance matrices: $P \defeq \mathcal{N}(\mu_P, \Sigma_P)$ and $Q \defeq \mathcal{N}(\mu_Q, \Sigma_Q)$ is roughly equal to the $\mmd$ distance computed using samples from $\mathcal{U}_{\alpha}(P) \defeq [\text{~Uniform}(\mu_P^1 \pm \alpha \Sigma_{P}^{1,1})] \times \cdots \times [\text{~Uniform}(\mu_P^d \pm \alpha \Sigma_{P}^{d,d})]$ and $Q$. This means that when minimizing the $\mmd$ distance to train distribution $Q$, the gradient signal would push $Q$ towards a uniform distribution supported on $P$'s support as this solution exhibits a lower MMD value -- which is the objective we are optimizing.

Figure~\ref{fig:mmd} shows an empirical comparison of $\operatorname{MMD}(P, Q)$ when $Q = P$, computed by sampling $n$-samples from $P$, and $\operatorname{MMD}(\mathcal{U}_\alpha(P), Q)$ (also when $Q$ = $P$) computed by sampling $n$-samples from $\mathcal{U}_\alpha(P)$. We observe that $\mmd$ distance computed using limited samples can, in fact, be higher between a distribution and itself as compared to a uniform distribution over a distribution's support and itself. In Figure~\ref{fig:mmd}, note that for smaller values of $n$ and appropriately chosen $\alpha$ (mentioned against each figure, the support of the uniform distribution), the estimator for $\mmd(\mathcal{U}_{\alpha}(P), P)$ can provide lower estimates than the value of the estimator for $\mmd(P, P)$. This observation suggests that when the number of samples is not enough to sample infer the distribution shape, density-agnostic distances like MMD can be used as optimization objectives to push distributions to match supports. Subfigures (c) and (d) shows the increase in MMD distance as the support of the uniform distribution is expanded.

\begin{figure*}[h]
    \centering
    \begin{subfigure}[h]{0.49\textwidth}
        \centering
        \includegraphics[width=0.99\linewidth]{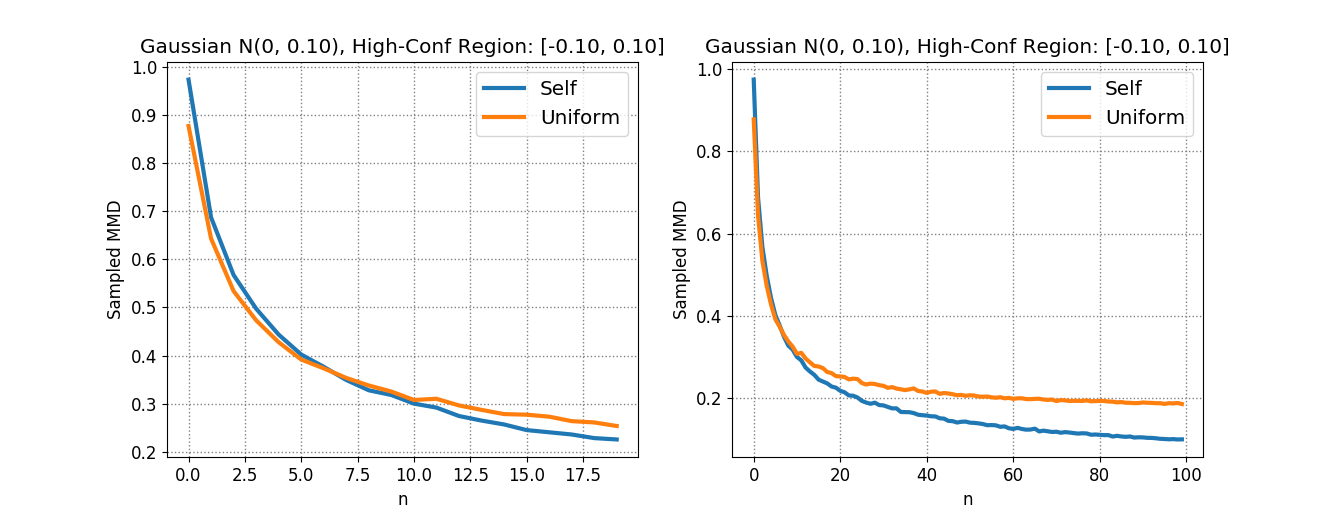}
        \caption{$\mathcal{N}(0, 0.1)$, $\mathcal{U}(-0.1, 0.1)$ }
    \end{subfigure}%
    ~ 
    \begin{subfigure}[h]{0.49\textwidth}
        \centering
        \includegraphics[width=0.99\linewidth]{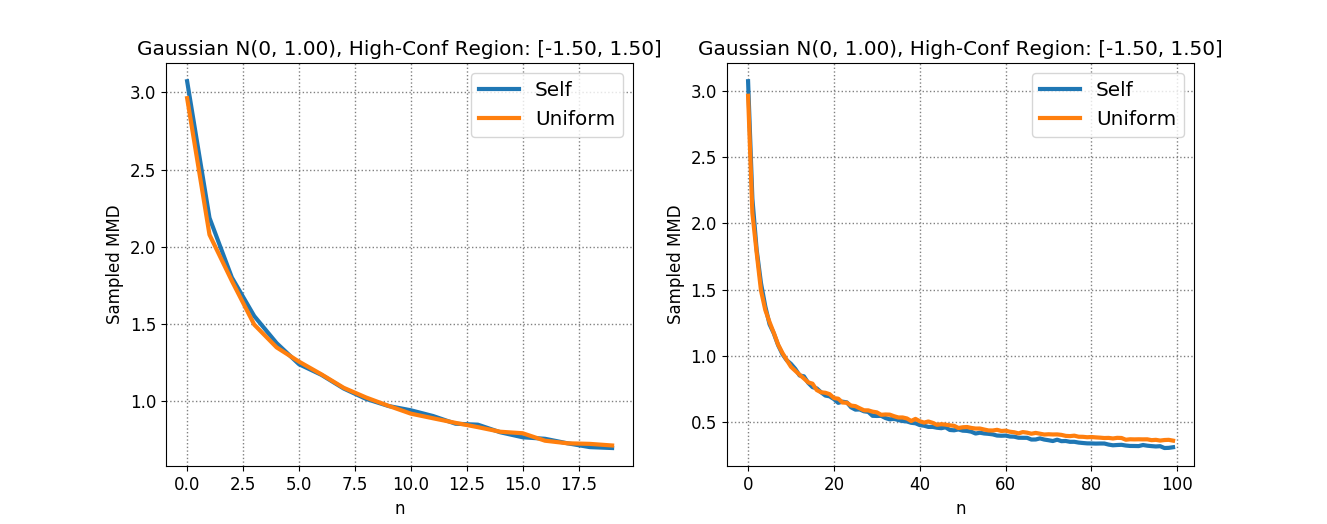}
        \caption{$\mathcal{N}(0, 1.0)$, $\mathcal{U}(-1.5, 1.5)$  }
    \end{subfigure}
    ~
    \begin{subfigure}[h]{0.49\textwidth}
        \centering
        \includegraphics[width=0.99\linewidth]{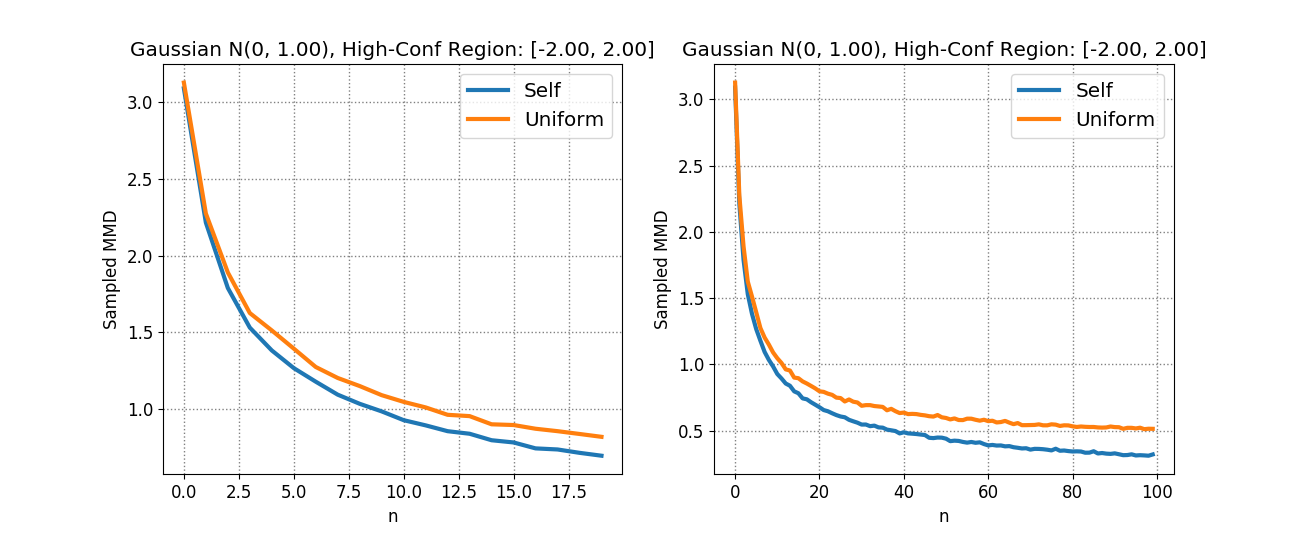}
        \caption{$\mathcal{N}(0, 1.0)$, $\mathcal{U}(-2.0, 2.0)$  }
    \end{subfigure}
    ~
    \begin{subfigure}[h]{0.49\textwidth}
        \centering
        \includegraphics[width=0.99\linewidth]{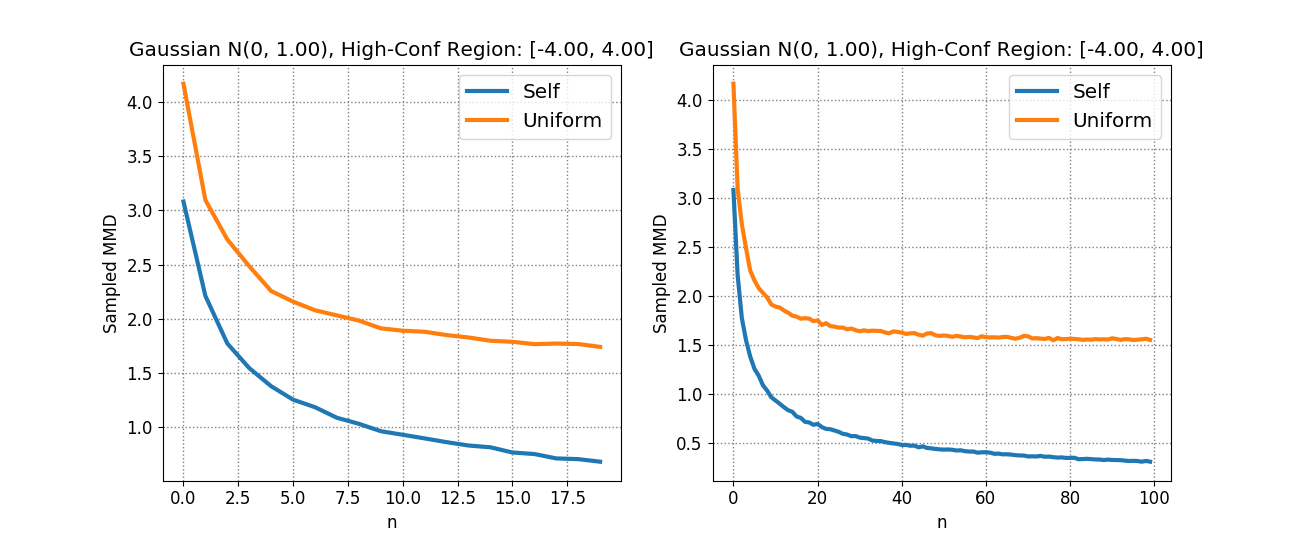}
        \caption{$\mathcal{N}(0, 1.0)$, $\mathcal{U}(-4.0, 4.0)$  }
    \end{subfigure}
    \caption{Comparing $\operatorname{MMD}$ distance between a $1$-d Gaussian distribution ($P$) and itself ($P$), and a uniform distribution over support set of the $P$ and the distribution $P$. The parameters of the Gaussian distribution ($P$) and the uniform distribution being considered are mentioned against each plot. ('Self' refers to $\mmd(P, P)$ and 'Uniform' refers to $\mmd(P, \mathcal{U}(P))$.) Note that for small values of $n \approx 1-10$, the $\mmd$ with the Uniform distribution is slightly lower in magnitude than the $\mmd$ between the distribution  $P$ and itself (sub-figures (a), (b) and (c)). For (d), as the support of this uniform distribution is enlarged, this leads to an increase in the value of $\operatorname{MMD}$ in the uniform approximation case -- which suggests that a near-local minimizer for the $\mmd$ distance can be obtained by making sure that the distribution which is being trained in this process shares the same support as the other given distribution.}
    \label{fig:mmd}
\end{figure*}

In order to provide a theoretical example, we refer to Example 1 in \citet{gretton2012kernel}, and extend it. First, note that the example argues that a fixed sample size of samples drawn from a distribution $P$, there exists another discrete distribution $Q$ supported on $m^2$ samples from the support set of $P$, such that there atleast is a probability $\left(\begin{array}{c}{m^{2}} \\ {m}\end{array}\right) \frac{m !}{m^{2 m}}>1-e^{-1}>0.63$ that a sample from $Q$ is indeed a sample from $P$ as well. So, with a smaller value of $m$, \textit{no} {2-sample test} will be able to distinguish between $P$ and $Q$. We would also note that this example is exactly the argument that our algorithm build upon. We further extend this example by noting that if $Q$ were rather not completely supported on the support of $P$, then there exists atleast a probability of $\epsilon$ that a sample from $Q$ lies outside the support of $P$. This gives us a lower bound on the value of the $\mmd$ estimator, indicating that the $\mmd$ 2-sample test will be able to detect this distribution due to an irreducible difference of $\epsilon \sqrt{\min_{y \in \text{Extremal(P)}} \mathbb{E}_{x \sim P}[k(x, y)]}$ (where $y$ is an "extremal point" in $P$'s support) in the MMD estimate.             

\section{Additional Experimental Details}
\label{app:additional_details}
\paragraph{Data collection} We trained behaviour policies using the Soft Actor-Critic algorithm~\cite{haarnoja2018sac}. In all cases, random data was generated by running a uniform at random policy in the environment. Optimal data was generated by training SAC agents in all 4 domains until convergence to the returns mentioned in Figure~\ref{fig:optimal_random}. Mediocre data was generated by training a policy until the return value marked in each of the plots in Figure~\ref{fig:mediocre}. Each of our datasets contained 1e6 samples. We used the same datasets for evaluating different algorithms to maintain uniformity across results.

\paragraph{Choice of kernels} In our experiments, we found that the choice of the kernel is an important design decision that needs to be made. In general, we found that a Laplacian kernel $k(x, y) = \exp(\frac{-||x - y||}{\sigma})$ worked well in all cases. Gaussian kernel $k(x, y) =\exp(\frac{-||x - y||^2}{2 \sigma^2})$ worked quite well in the case of optimal dataset. For the Laplacian kernel, we chose $\sigma = 10.0$ for Cheetah, Ant and Hopper, and $\sigma=20.0$ for Walker. However, we found that $\sigma=20.0$ worked well for all environments in all settings. For the Gaussian kernel, we chose $\sigma=20.0$ for all settings. Kernels often tend to not provide relevant measurements of distance especially in high-dimensional spaces, so one direction for future work is to design right kernels. We further experimented with a mixture of Laplacian kernel with different bandwidth parameters $\sigma$ ($1.0, 10.0, 50.0$) on Hopper-v2 and Walker2d-v2 where we found that it performs comparably and sometimes is better than a simple Laplacian kernel, probably because it is able to track supports upto different levels of thresholds due to multiple kernels.   

\paragraph{More details about the algorithm} At evaluation time, we find that using the greedy maximum of the Q-function over the support set of the behaviour policy (which can be approximated by sampling multiple Dirac-delta policies $\delta_{a_i}$ from the policy $\pi_\phi$ and performing a greedy maximization of the Q-values over these Dirac-delta policies) works best, better than unrolling the learned actor $\pi_\phi$ in the environment. This was also found useful in \cite{fujimoto2018off}. Another detail about the algorithm is deciding which samples to use for computing the $\mmd$ objective. We train a parameteric model $\pi_{data}$ which fits a tanh-Gaussian distribution to $a$ given the states $s$, $\pi_{data}(\cdot|s) = \tanh{\mathcal{N}(\mu(\cdot|s), \sigma(\cdot|s))}$ and then use this to sample a candidate $n$ actions for computing the MMD-distance, meaning that MMD is computed between $a_1, \cdots, a_N \sim \pi_{data}$ and $\pi_\phi$. We find the latter to work better in practice. Also, computing the $\mmd$ distance between actions before applying the tanh transformation work better, and leads to a constraint, that perhaps provides stronger gradient signal -- because tanh saturates very quickly, after which gradients almost vanish. 

\paragraph{Other hyperparameters} Other hyperparameters include the following -- (1) The variance of the Gaussian $\sigma^2$ /(standard deviation of) Laplacian kernel $\sigma$: We tried a variance of 10, 20, and 40. We found that 10 and 20 worked well across Cheetah, Hopper and Ant, and 20 worked well for Walker2d; (2) The learning rate for the Lagrange multiplier was chosen to be 1e-3, and the $\log$ of the Lagrange multiplier was clipped between $[-5, 10]$ to prevent instabilities; (3) For the policy improvement step, we found using average Q works better than min Q for Walker2d. For the baselines, we used BCQ code from the official implementation accompanying~\cite{fujimoto2018off}, TD3 code from the official implementation accompanying~\cite{fujimoto18addressing} and the BC baseline was the VAE-based behaviour cloning baseline also used in \cite{fujimoto2018off}. We evaluated on 10 evaluation episodes (which were separate from the train distribution) after every 1000 iterations and used the average score and the variance for the plots. 

\section{Additional Experimental Results}
\label{exp:additional_results}

\label{app:q_vs_mc}
\begin{figure*}[h]
    \centering
    \begin{subfigure}[h]{0.31\textwidth}
        \centering
        \includegraphics[width=0.99\linewidth]{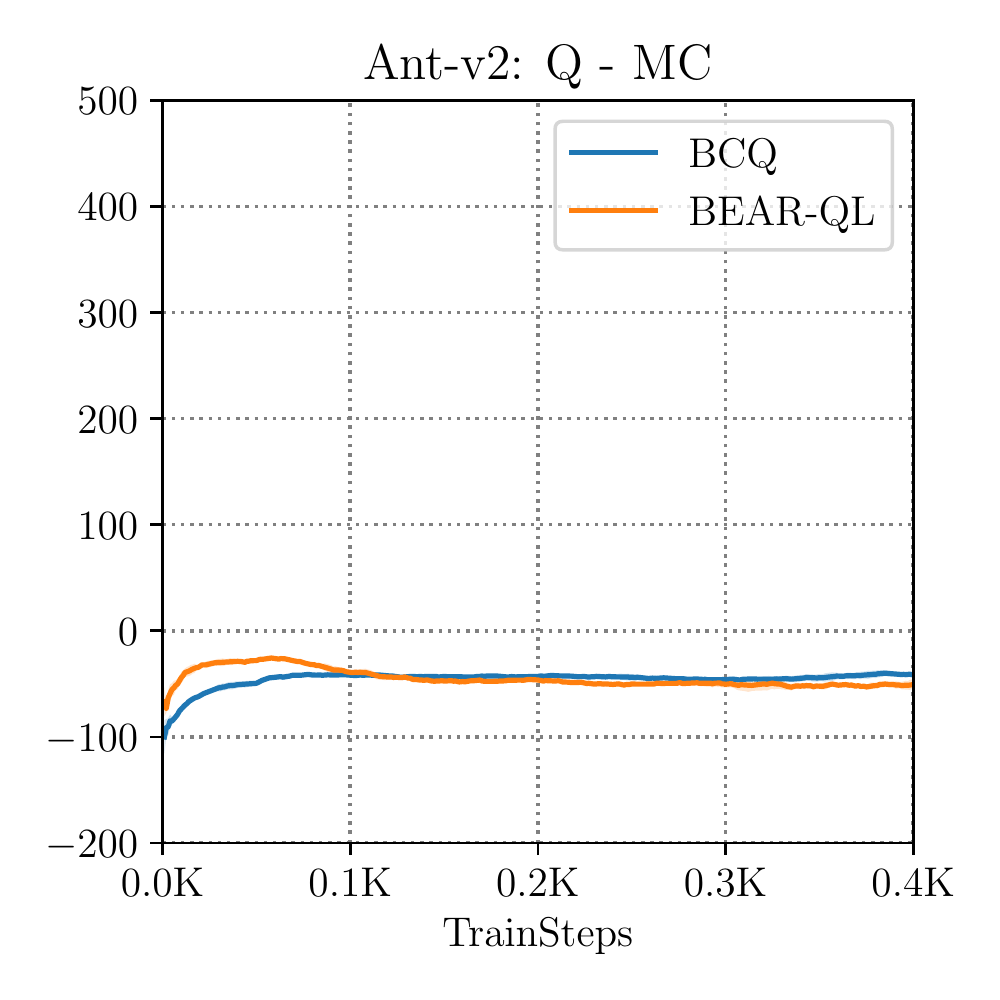}
    \end{subfigure}%
    ~
    \begin{subfigure}[h]{0.31\textwidth}
        \centering
        \includegraphics[width=0.99\linewidth]{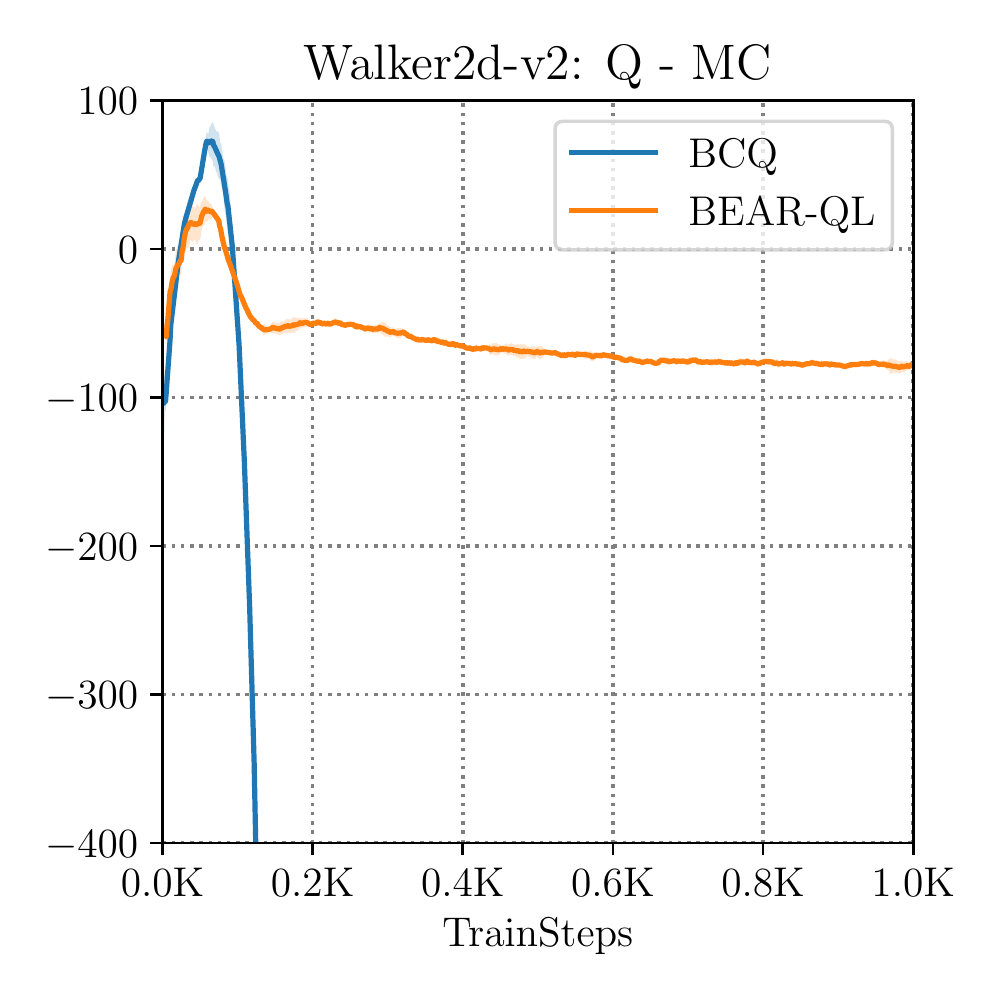}
    \end{subfigure}%
    \caption{The trend of the difference between the Q-values and Monte-Carlo returns: $Q - MC$ returns for 2 environments. Note that a high value of $Q-MC$ corresponds to more overestimation. In these plots, BEAR-QL is more well behaved than BCQ. In Walker2d-v2, BCQ tends to diverge in the negative direction. In the case of Ant-v2, although roughly the same, the difference between Q values and Monte-carlo returns is slightly lower in the case of BEAR-QL suggestion no risk of overestimation. (This corresponds to medium-quality data.)}
    \label{fig:q_mc}
\end{figure*}


\label{app:q_val_compare}
\begin{figure*}[h]
    \centering
    \begin{subfigure}[h]{0.31\textwidth}
        \centering
        \includegraphics[width=0.99\linewidth]{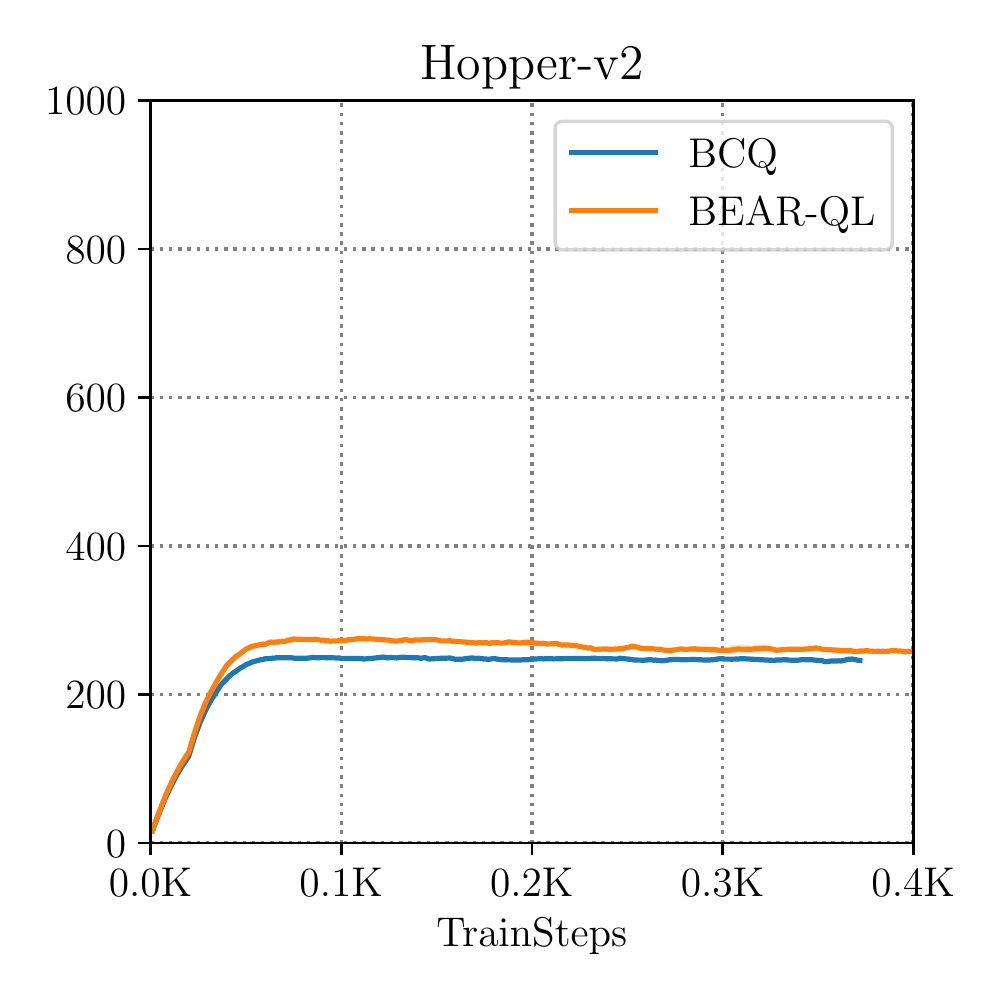}
    \end{subfigure}%
    ~
    \begin{subfigure}[h]{0.31\textwidth}
        \centering
        \includegraphics[width=0.99\linewidth]{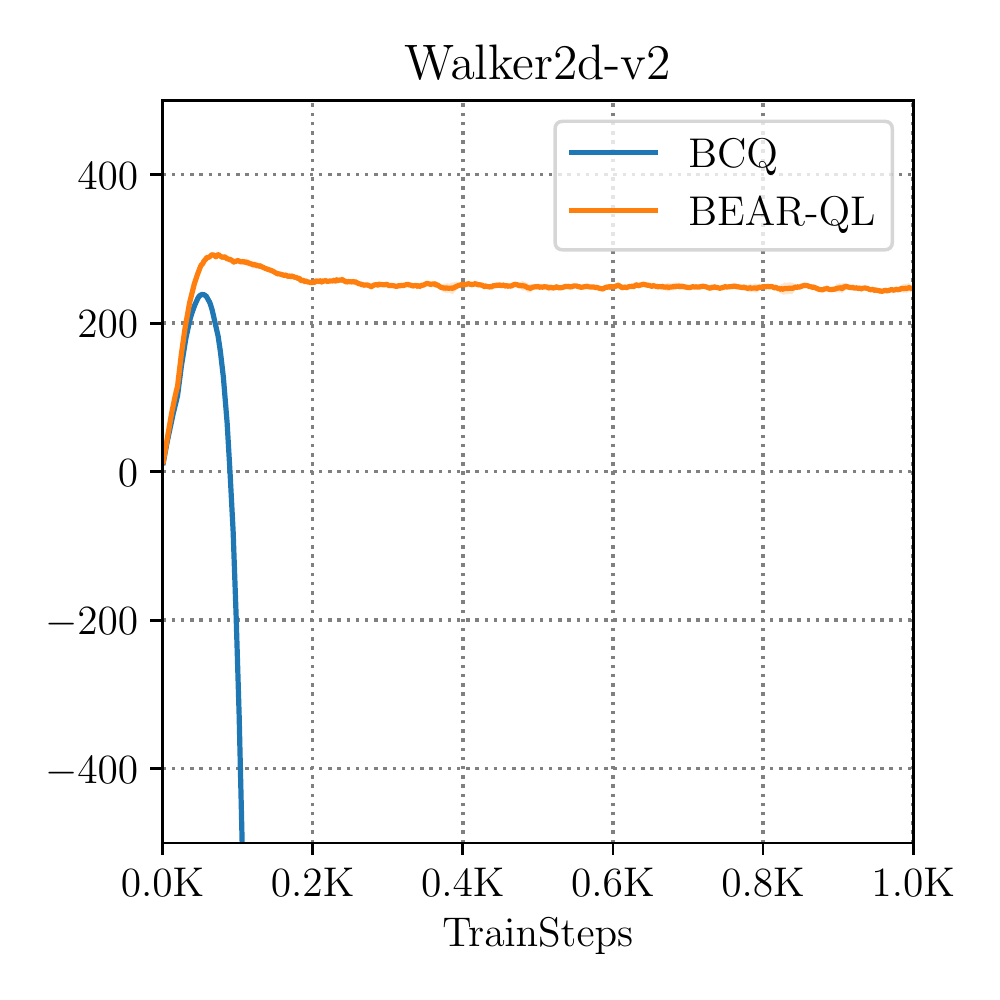}
    \end{subfigure}%
    ~
    \begin{subfigure}[h]{0.31\textwidth}
        \centering
        \includegraphics[width=0.99\linewidth]{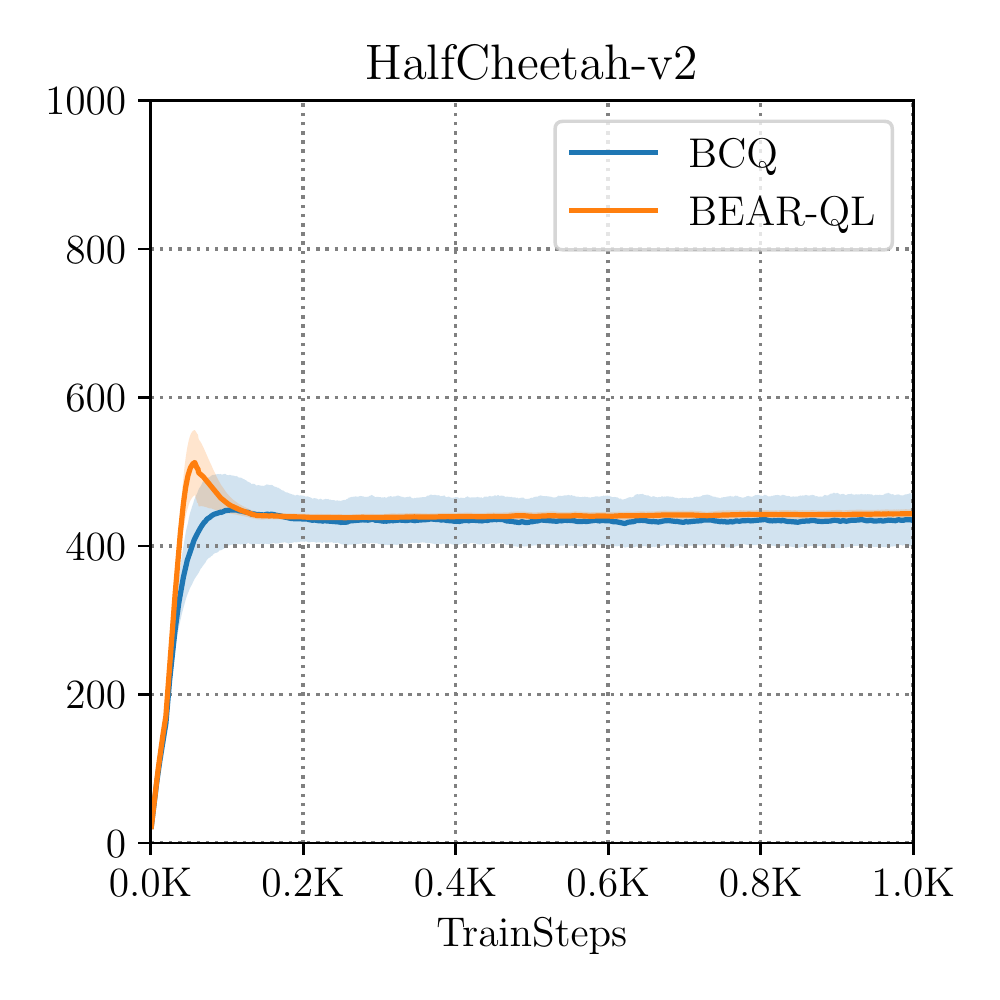}
    \end{subfigure}%
    \caption{The trends of Q-values as a function of number of gradient steps taken in case of 3 environments. BCQs Q-values tend to be more unstable (especially in the case of Walker2d, where they diverge in the negative direction) as compared to BEAR-QL. This corresponds to medium-quality data.}
    \label{fig:q_val_mediocre}
\end{figure*}

In this section, we provide some extra plots for some extra experiments. In Figure~\ref{fig:q_mc} we provide the difference between learned Q-values and Monte carlo returns of the policy in the environment. In Figure~\ref{fig:q_val_mediocre} we provide the trends of comparisons of Q-values learned by BEAR-QL and BCQ in three environments. In Figure~\ref{fig:kl_vs_mmd_single} we compare the performance when using the MMD constraint vs using the KL constraint in the case of three environments. 
In order to be fair at comparing to MMD, we train a model for the behaviour policy and constrain the KL-divergence to this behaviour policy. (For MMD, we compute MMD using samples from the model of the behaviour policy.) Note that in the case of Half Cheetah with medium-quality data, KL divergence constraint works pretty well, but it fails drastically in the case of Hopper and Walker2d and the Q-values tend to diverge. Figure~\ref{fig:kl_vs_mmd_single} summarizes the trends for 3 environments.

\begin{figure*}[t]
    \centering
    \begin{subfigure}[t]{0.31\textwidth}
        \centering
        \includegraphics[width=0.99\linewidth]{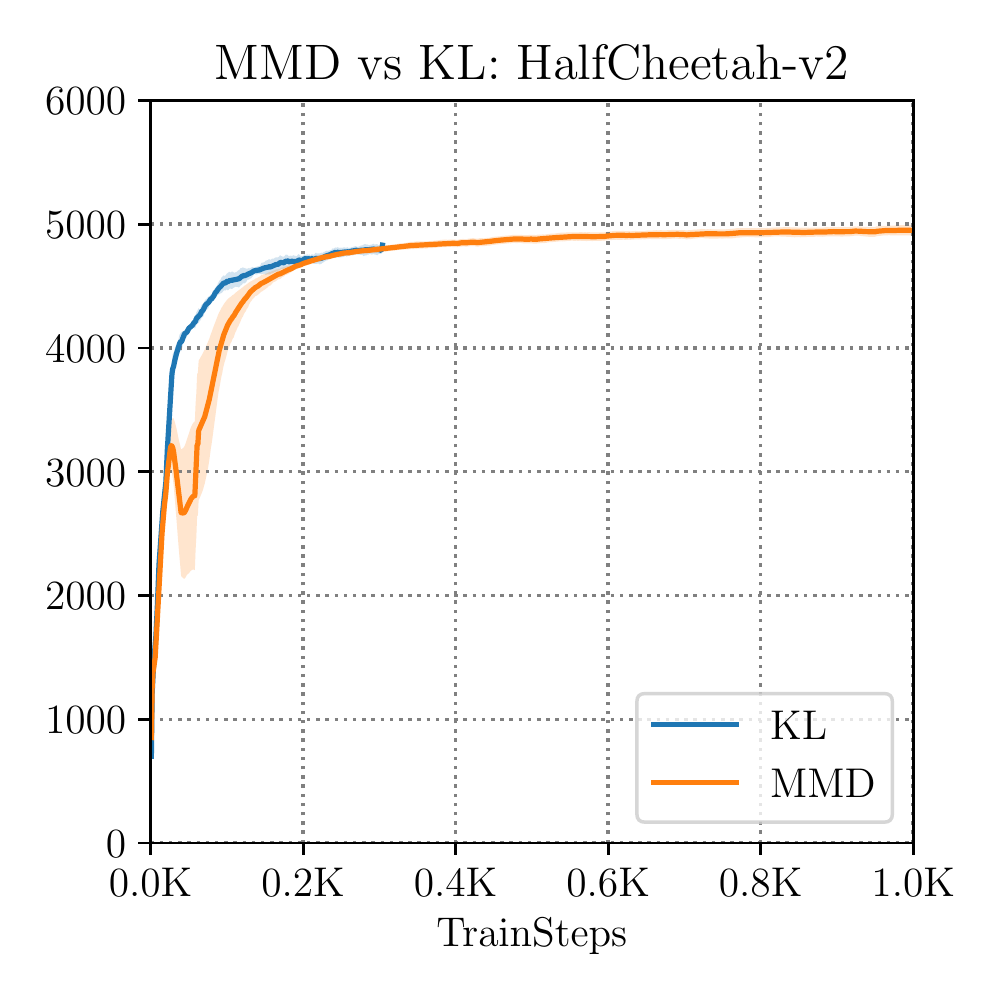}
    \end{subfigure}%
    ~
    \begin{subfigure}[t]{0.31\textwidth}
        \centering
        \includegraphics[width=0.99\linewidth]{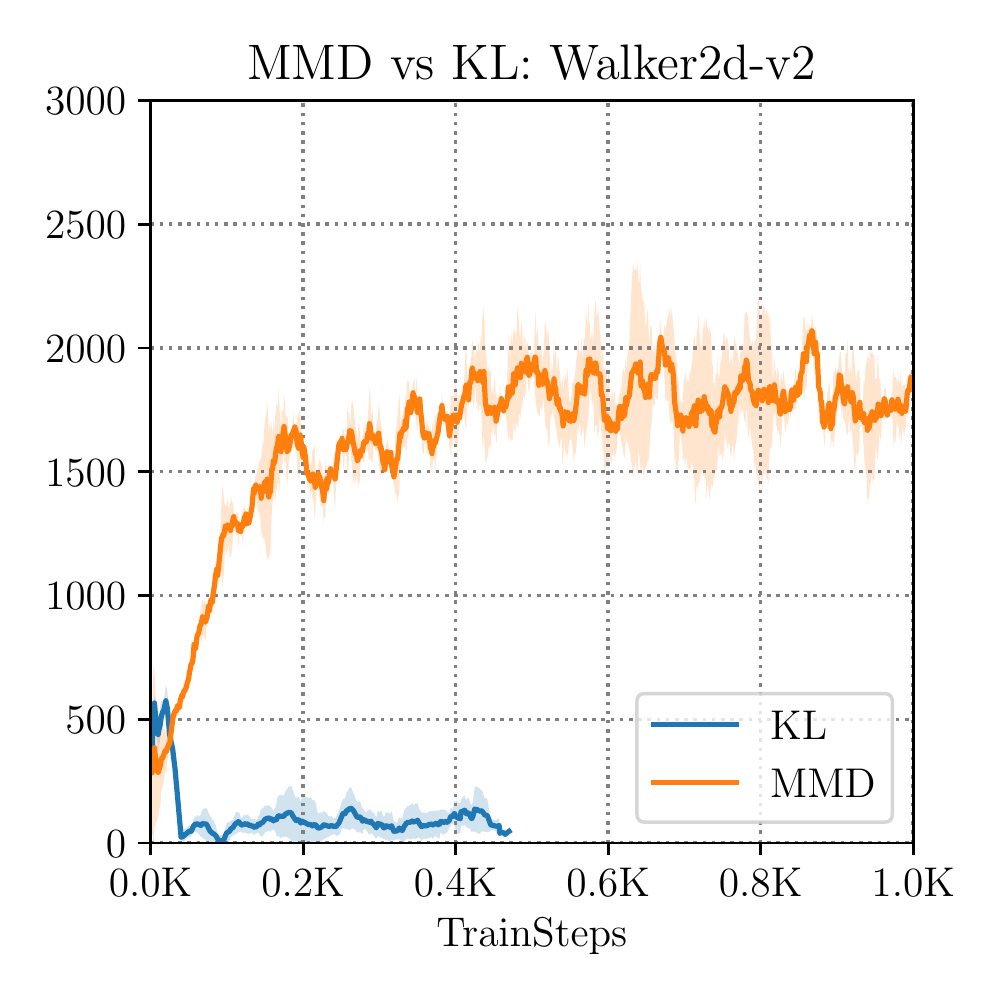}
    \end{subfigure}%
    ~
    \begin{subfigure}[t]{0.31\textwidth}
        \centering
        \includegraphics[width=0.99\linewidth]{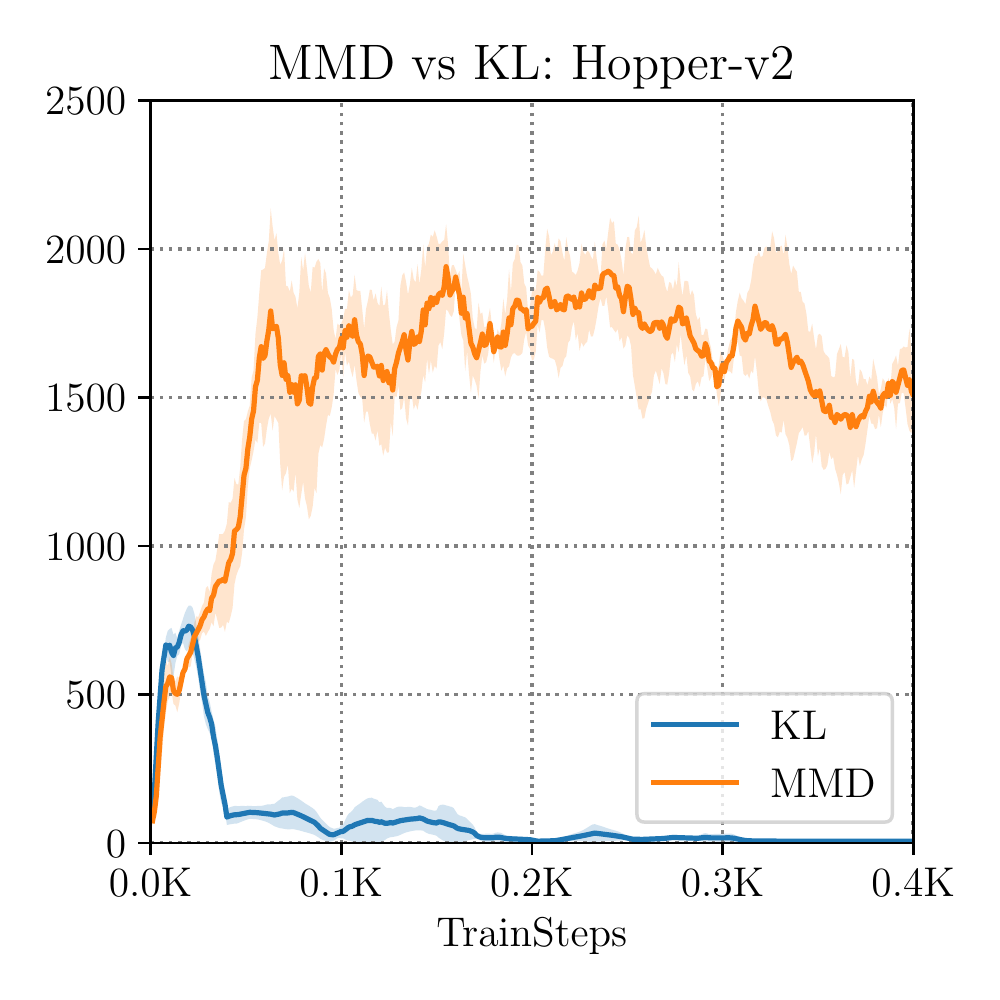}
    \end{subfigure}%
    \caption{Performance Trends (measured in AverageReturn) for Hopper-v2, HalfCheetah-v2 and Walker2d-v2 environments with BEAR-QL algorthm but varying kind of constraint. In general we find that using the KL constraint leads to worse performance. However, in some rare cases (for example, HalfCheetah-v2), the KL constraint learns faster. In general, we find that the KL-constraint often leads to diverging Q-values. This experiment corresponds to medium-quality data.}
    \label{fig:kl_vs_mmd_single}
\end{figure*}

We further study the performance of the KL-divergence in the setting when the KL-divergence is stable. In this setting we needed to perform extensive hyperparameter tuning to find the optimal Lagrange multiplier for the KL-constraint and plain and simple dual descent always gave us an unstable solution with the KL-constraint. Even in this case tuned hyperparameter case, we find that using a KL-constraint is worse than using a MMD-constraint. Trends are summarized in Figure~\ref{fig:tuned_kl}. 
\begin{figure*}[t]
    \centering
    \begin{subfigure}[t]{0.31\textwidth}
        \centering
        \includegraphics[width=0.99\linewidth]{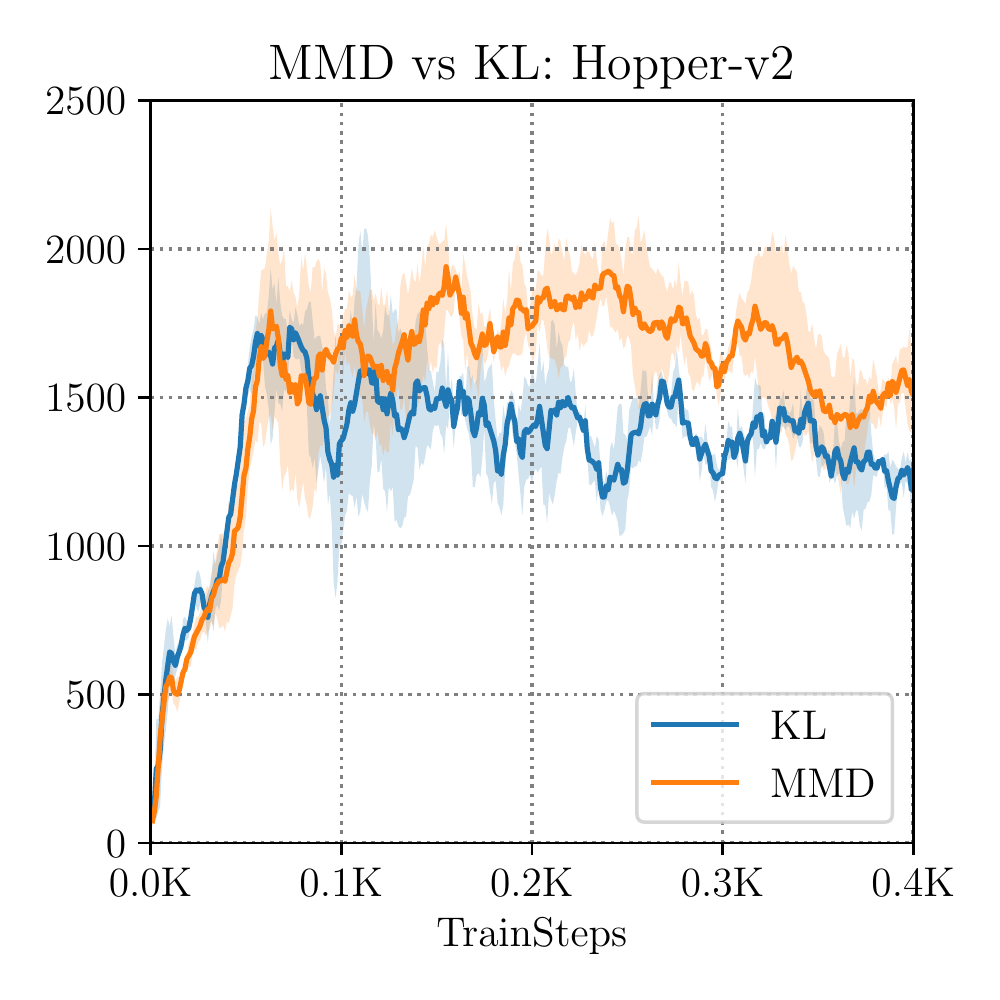}
    \end{subfigure}%
    ~
    \begin{subfigure}[t]{0.31\textwidth}
        \centering
        \includegraphics[width=0.99\linewidth]{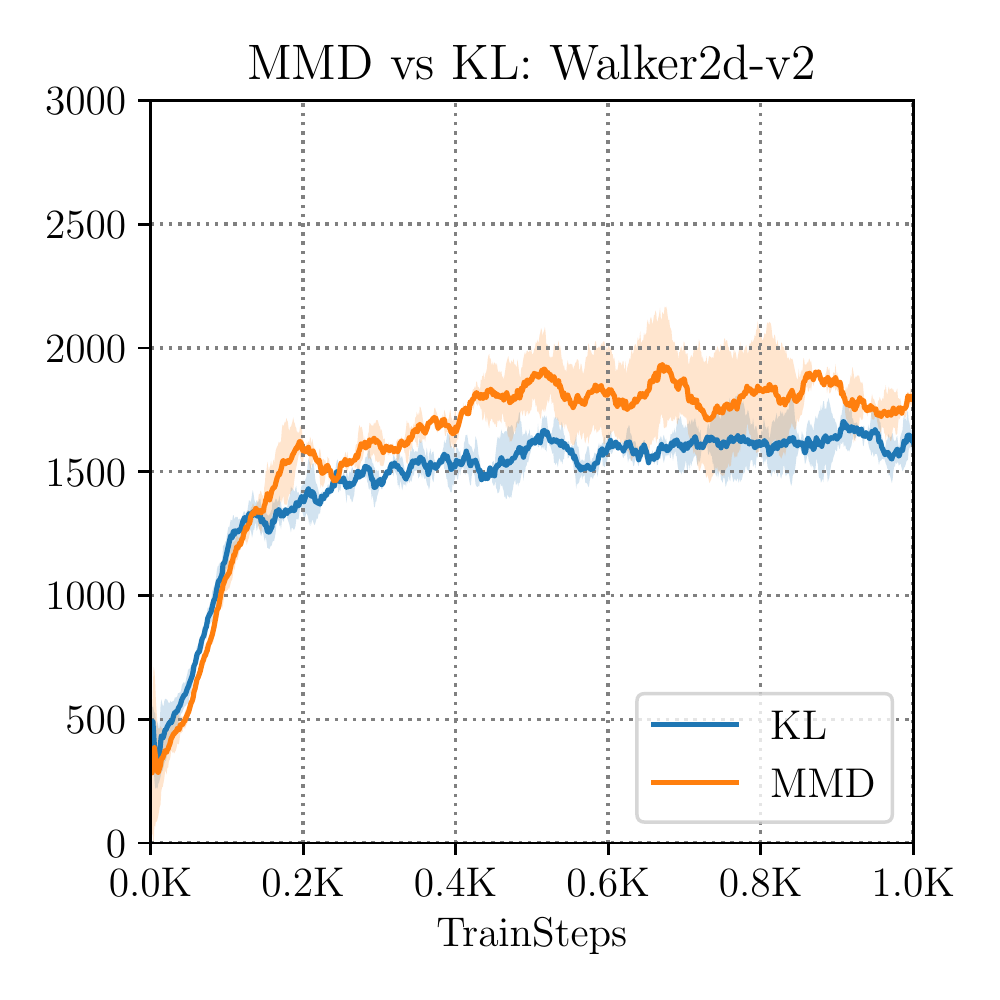}
    \end{subfigure}%
    \caption{Performance Trends (measured in Average Returns) for Hopper-v2 and Walker2d-v2 environments with BEAR-QL algorithm with an extensively tuned KL-constraint and the MMD-constraint from. Note that the MMD-constraint still outperforms the KL-constraint.}
    \label{fig:tuned_kl}
\end{figure*}

As described in Section~\ref{app:bearql-more}, we can achieve a reduced overall error $||V_k(s) - V^*(s)||$, if we use the MMD support-matching constraint alongside importance sampling, i.e. when we multiply the Bellman error with the inverse of the behaviour policy density. Empirically, we tried reweighting the Bellman error by inverse of the fitted behavior policy density, alongside the BEAR-QL algorithm. The trends for two environments and medium-quality data are summarized in Figure~\ref{fig:is}. We found that reweighting the Bellman error wasn't that useful, although in theory, it provides an absolute error reduction as described by Theorem 4.1. We hypothesize that this could be due to the possible reason that when optimizing neural nets using stochastic gradient procedures, importance sampling isn't that beneficial~\citep{byrd19is}.

\begin{figure}
    \centering
    \includegraphics[scale=0.3]{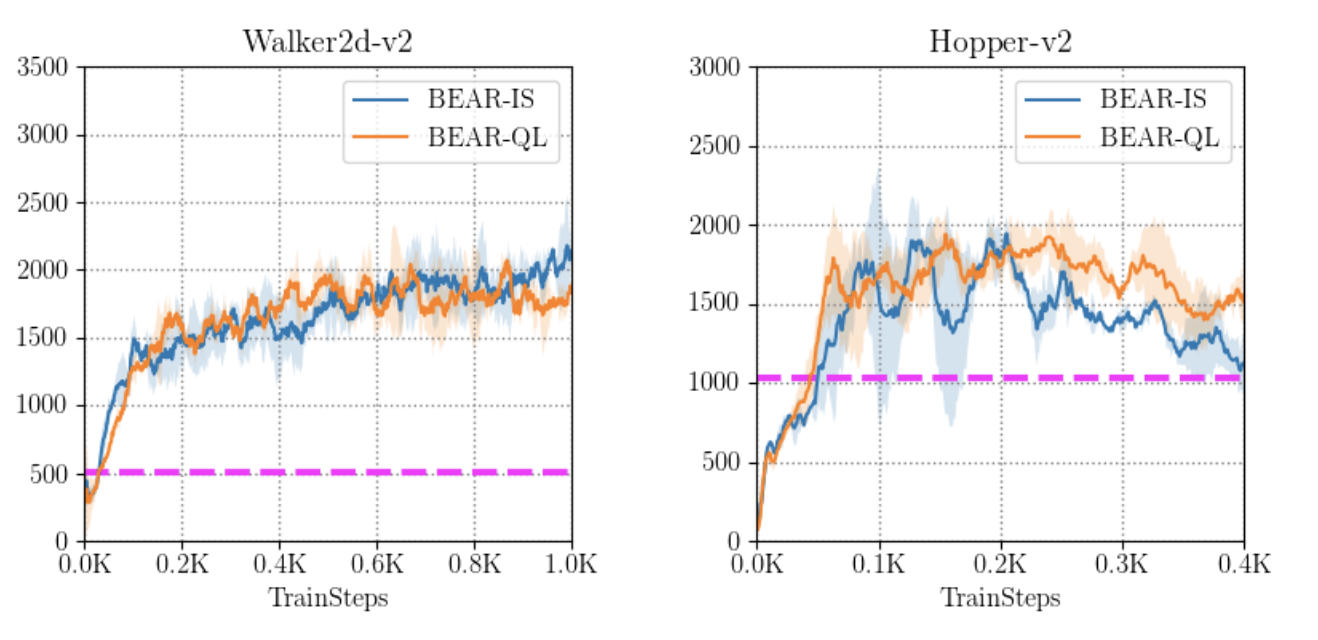}
    \caption{BEAR with importance sampled Bellman error minimization. We find that importance sampling isn't that beneficial in practice.}
    \label{fig:is}
\end{figure}


\end{document}